\documentclass[
  smallextended,
  envcountsect,
]{svjour3}
\smartqed  
\usepackage{graphicx}
\usepackage{subfigure}

\usepackage{times}  
\usepackage{helvet} 
\usepackage{courier}  
\usepackage{graphicx} 
\usepackage{wrapfig}
\usepackage{booktabs}       
\usepackage{amsfonts}       
\usepackage{nicefrac}       
\usepackage{microtype}      
\usepackage{dblfloatfix}
\usepackage{pifont}
\usepackage{array}\usepackage{makecell} 
\usepackage{color}
\usepackage{mathrsfs}
\usepackage{amssymb}
\usepackage{amsmath}
\usepackage[page,header]{appendix}
\usepackage{dsfont}
\usepackage{hyperref}
\usepackage[capitalize]{cleveref}
\usepackage{mathtools} 
\usepackage{enumitem}
\usepackage{algorithm}
\usepackage{algorithmic}
\usepackage{multirow}
\usepackage{natbib}

\definecolor{darkgreen}{rgb}{0,0.5,0}
\definecolor{darkred}{rgb}{0.8,0,0}
\definecolor{purple}{rgb}{0.643,0.259,0.859}
\DeclareMathOperator*{\argmax}{arg\,max}
\DeclareMathOperator*{\argmin}{arg\,min}
\graphicspath{ {./images/} }
\DeclarePairedDelimiter\ceil{\lceil}{\rceil}
\newcommand{\norm}[1]{\left\lVert#1\right\rVert}
\usepackage{times}
\newcommand{\cmark}{\ding{51}}%
\newcommand{\xmark}{\ding{55}}%

\usepackage{subfigure}

\newcommand{\stav}[1]{\textcolor{blue}{\{Stav\}}}
\journalname{Inverse Reinforcement Learning in Contextual MDPs}
\begin{document}

\title{Inverse Reinforcement Learning in Contextual MDPs 
}

\titlerunning{Inverse Reinforcement Learning in Contextual MDPs}        

\author{Stav Belogolovsky* \and
Philip Korsunsky* \and
Shie Mannor \and Chen Tessler \and Tom Zahavy
}


\institute{S. Belogolovsky \at
          \email{stav.belo@gmail.com}
          \and
          P. Korsunsky \at
          \email{philip.korsunsky@gmail.com}
          \and
          * Equal contribution.
}

\date{Received: date / Accepted: date}

\maketitle

\begin{abstract}
We consider the task of Inverse Reinforcement Learning in Contextual Markov Decision Processes (MDPs). In this setting, contexts, which define the reward and transition kernel, are sampled from a distribution. In addition, although the reward is a function of the context, it is not provided to the agent. Instead, the agent observes demonstrations from an optimal policy. The goal is to learn the reward mapping, such that the agent will act optimally even when encountering previously unseen contexts, also known as zero-shot transfer. We formulate this problem as a non-differential convex optimization problem and propose a novel algorithm to compute its subgradients. Based on this scheme, we analyze several methods both theoretically, where we compare the sample complexity and scalability, and empirically. Most importantly, we show both theoretically and empirically that our algorithms perform zero-shot transfer (generalize to new and unseen contexts). Specifically, we present empirical experiments in a dynamic treatment regime, where the goal is to learn a reward function which explains the behavior of expert physicians based on recorded data of them treating patients diagnosed with sepsis.

\end{abstract}

\section{Introduction}\label{intro}
\begin{figure*}[t]
\centering
\includegraphics[width=0.9\linewidth]{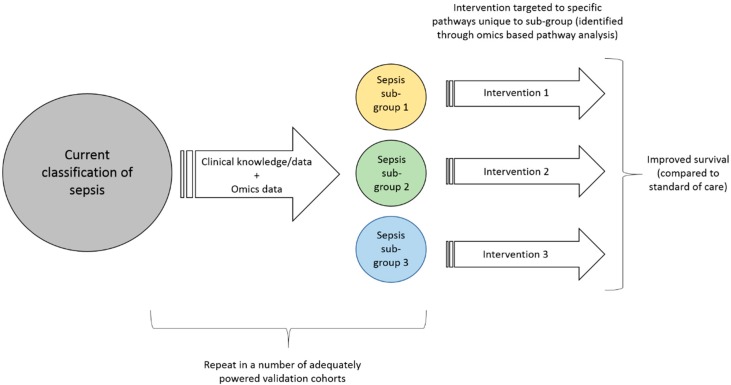}
\caption{Personalized medicine in sepsis treatment. Credit: \cite{itenov2018sepsis}.}
\label{fig: personalized medicine}
\end{figure*}

\noindent
Real-world sequential decision making problems often share three important properties --- (1) \emph{the reward function is often unknown}, yet (2) expert demonstrations can be acquired, and (3) the reward and/or dynamics often depend on a \emph{static parameter}, also known as the context. For a concrete example, consider a dynamic treatment regime \citep{chakraborty2014dynamic}, where a clinician acts to improve a patient's medical condition. While the patient's dynamic measurements, e.g., heart rate and blood pressure, define the state, there are static parameters, e.g., age and weight, which determine how the patient reacts to certain treatments and what form of treatment is optimal.\\

\noindent
The contextual model is motivated by recent trends in personalized medicine, predicted to be one of the technology breakthroughs of 2020 by MIT's Technology Review \citep{mit_technology_review_2020}. As opposed to traditional medicine, which provide a treatment for the ``average patient", in the contextual model, patients are separated into different groups for which the medical decisions are tailored (\cref{fig: personalized medicine}). This enables the decision maker to provide tailored decisions (e.g., treatments) which are more effective, based on these static parameters. \\

\noindent
For example, in \citet{wesselink2018intraoperative}, the authors study organ injury, which may occur when a specific measurement (mean arterial pressure) decreases below a certain threshold. They found that this threshold varies across different patient groups (contextual behavior). In other examples, clinicians set treatment goals for the patients, i.e., they take actions to drive the patient measurements towards some predetermined values. For instance, in acute respiratory distress syndrome (ARDS), clinicians argue that these treatment goals should depend on the static patient information (the context) \citep{berngard2016personalizing}.\\

\noindent
In addition to the contextual structure, we consider the setting where the reward itself is unknown to the agent. This, also, is motivated by real-world problems, in which serious issues may arise when manually attempting to define a reward signal. For instance, when treating patients with sepsis, the only available signal is the mortality of the patient at the end of the treatment \citep{komorowski2018artificial}. While the goal is to improve the patients' medical condition, minimizing mortality does not necessarily capture this objective. This model is illustrated in \cref{fig: coirl diagram}. The agent observes expert interactions with the environment, either through pre-collected data, or through interactive expert interventions. The agent then aims to find a reward which \emph{explains} the behavior of the expert, meaning that the experts policy is optimal with respect to this reward.\\

\noindent
To tackle these problems, we propose the \textbf{Contextual Inverse Reinforcement Learning (COIRL)} framework. Similarly to Inverse Reinforcement Learning \citep[IRL]{ng2000algorithms}, provided expert demonstrations, the goal in COIRL is to learn a reward function which explains the expert's behavior, i.e., a reward function for which the expert behavior is optimal. In contrast to IRL, in COIRL the reward is not only a function of the state features but also the context. Our aim is to provide theoretical analysis and insights into this framework. As such, throughout most of the paper we consider a reward which is linear in both the context and the state features. This analysis enables us to propose algorithms, analyze their behavior and provide theoretical guarantees. We further show empirically in \cref{sec:experiments} that our method can be easily extended to mappings which are non-linear in the context using deep neural nets.  \\

\begin{figure*}[t]
\centering
\includegraphics[width=0.6\linewidth]{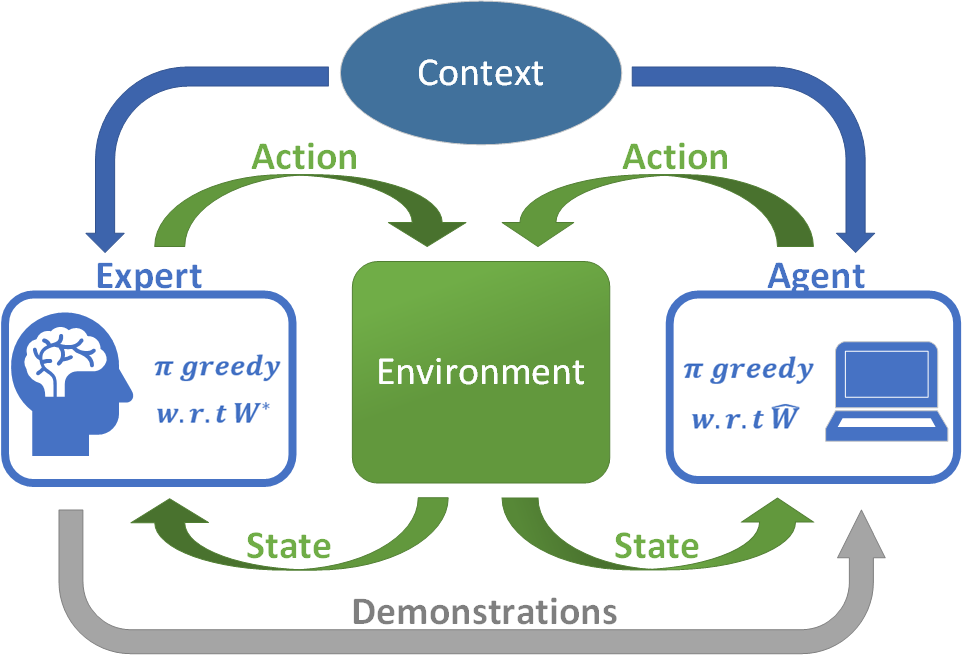}
\caption{The COIRL framework: a context vector parametrizes the environment. For each context, the expert uses the true mapping from contexts to rewards$, W^*,$ and provides demonstrations. The agent learns an estimation of this mapping $\hat{W}$ and acts optimally with respect to it.}
\label{fig: coirl diagram}
\end{figure*}

\noindent
The paper is organized as follows. In \cref{sec:prelim} we introduce the Contextual MDPs and provide relevant notation. In \cref{sec:mda} we formulate COIRL, with a linear mapping, as a convex optimization problem. We show that while this loss is not differentiable, it can be minimized using subgradient descent and provide methods to compute subgradients. We propose algorithms based on Mirror Descent (MDA) and Evolution Strategies (ES) for solving this task and analyze their sample complexity. In addition, in \cref{subsec:ellips_method}, we adapt the cutting plane (ellipsoid) method to the COIRL domain. In \cref{sec:related_work} we discuss how existing IRL approaches can be applied to COIRL problems and their limitations. Finally, in \cref{sec: transfer} we discuss how to efficiently (without re-solving the MDP) perform zero-shot transfer to unseen contexts. \\

\noindent 
These theoretical approaches are then evaluated, empirically, in \cref{sec:experiments}. We perform extensive testing of our methods and the relevant baselines both on toy problems and on a dynamic treatment regime, which is constructed from real data. We evaluate the run-time of IRL vs COIRL, showing that when the structure is indeed contextual, standard IRL schemes are computationally inefficient. We show that COIRL is capable of generalizing (zero-shot transfer) to unseen contexts, while behavioral cloning (log-likelihood action matching) is sub-optimal and struggles to find a good solution. These results show that in contextual problems, COIRL enables the agent to quickly recover a reward mapping that explains the expert's behavior, outperforming previous methods across several metrics and can thus be seen as a promising approach for real-life decision making.\\

\noindent
Our contribution is three fold: First, the formulation of COIRL problem as a convex optimization problem, and the novel adaptation of the descent methods to this setting. Second, we provide theoretical analysis for the \emph{linear} case for all of the proposed methods. Third, we bridge between the theoretical results and real-life application through a series of experiments that aim to apply COIRL to sepsis treatment (\cref{sec:experiments}).

\section{Preliminaries}\label{sec:prelim}
\textbf{Contextual MDPs:}
A Markov Decision Process \cite[\textbf{MDP}]{puterman1994markov} is defined by the tuple $(\mathcal{S},\mathcal{A},P,\xi,R,\gamma)$ where $\mathcal{S}$ is a finite state space, $\mathcal{A}$ a finite action space, $P : S \times S \times A \to [0,1]$ the transition kernel, $\xi$ the initial state distribution, $R: \mathcal{S} \to \mathbb{R}$ the reward function and $\gamma\in[0,1)$ is the discount factor. A Contextual MDP \cite[\textbf{CMDP}]{hallak2015contextual} is an extension of an MDP, and is defined by $(\mathcal{C},\mathcal{S},\mathcal{A},\mathcal{M},\gamma)$ where $\mathcal{C}$ is the context space, and $\mathcal{M}$ is a mapping from contexts $c \in \mathcal{C}$ to MDPs: $\mathcal{M}(c) = (\mathcal{S}, \mathcal{A}, P_c, \xi, R_c, \gamma)$. For consistency with prior work, we consider the discounted infinite horizon scenario. We emphasize here that all the results in this paper can be easily extended to the episodic finite horizon and the average reward criteria.\\

\noindent
We consider a setting in which each state is associated with a feature vector $\phi : \mathcal{S} \rightarrow[0,1]^k$, and the reward for context $c$ is a linear combination of the state features: $R^*_c(s) = f^*(c)^T\phi(s)$. The goal is to approximate $f^*(c)$ using a function $f_W(c)$ with parameters $W$. This notation allows us to present our algorithms for any function approximator $f_W(c)$, and in particular a deep neural network (DNN).\\

\noindent
For the theoretical analysis, we will further assume a \textit{linear setting}, in which the reward function and dynamics are linear in the context. Formally:
$$f^*(c)= c^TW^*,\:  f_W (c) = c^T W,\: W^* \in \mathcal{W},\: \text{and } P_c(s'|s,a) = c^T \begin{bmatrix} P_1(s'|s,a) \\ \vdots \\ P_d(s'|s,a) \end{bmatrix}$$
for some convex set $\mathcal{W}$. In order for the contextual dynamics to be well-defined, we assume the context space is the standard $d-1$ dimensional simplex: $\mathcal{C} = \Delta_{d-1}$. One interpretation of this model is that each row in the mapping $W^*$ along with the corresponding transition kernels defines a base MDP, and the MDP for a specific context is a convex combination of these base environments.\\

\noindent
We focus deterministic policies $\pi: \mathcal{S} \to \mathcal{A}$ which dictate the agent's behavior at each state. The value of a policy $\pi$ for context $c$ is:
\begin{equation*}
    V^\pi_c = E_{\xi,P_c,\pi}\left[\sum_{t=0}^{\infty} \gamma^t R^*_c(s_t)\right] =  f^*(c)^T\mu^\pi_c \enspace ,
\end{equation*}
where $\mu^\pi_c:=E_{\xi,P_c,\pi}[\sum_{t=0}^{\infty} \gamma^t \phi(s_t)]\in\mathbb{R}^k$ is called the \emph{feature expectations} of $\pi$ for context $c$. For other RL criteria there exist equivalent definitions of feature expectations; see  \citet{zahavy2019average} for the average reward.
We also denote by $V^\pi_c(s), \mu^\pi_c(s)$ the value and feature expectations for $\xi = \mathds{1}_s$. The action-value function, or the Q-function, is defined by: $Q_c^\pi(s,a) = R^*_c(s) + \gamma E_{s' \sim P_c(\cdot|s,a)}V^\pi_c(s')$. For the optimal policy with respect to (w.r.t.) a context $c$, we denote the above functions by $V^*_c, Q^*_c,\mu^*_c$. For any context $c$, $\pi^*_c$ denotes the optimal policy w.r.t. $R_c^*$, and $\hat{\pi}_c(W)$ denotes the optimal policy w.r.t. $\hat R_c(s) = f_W (c)^T\phi(s)$.\\

\noindent
For simpler analysis, we define a "flattening" operator, converting a matrix to a vector: $\mathbb{R}^{d\times k}\rightarrow \mathbb{R}^{d\cdot k}$ by $\underbar{$W$}=\begin{bmatrix} w_{1,1}, \hdots ,w_{1,k}, \hdots ,w_{d,1}, \hdots ,w_{d,k}\end{bmatrix}$. We also define the operator $\odot$ to be the composition of the flattening operator and the outer product: $u \odot v = \begin{bmatrix} u_1v_1, \hdots ,u_1v_k, \hdots ,u_dv_1, \hdots ,u_dv_k\end{bmatrix}$.
Therefore, the value of policy $\pi$ for context $c$ is given by $V^\pi_{c} = c^TW^*\mu^\pi_c = \underbar{$W^*$}^T(c\odot\mu^\pi_c),$ where $||c\odot\mu^\pi_c||_1 \leq \frac{k}{1-\gamma}$.\\

\subsection{Apprenticeship Learning and Inverse Reinforcement Learning}

In Apprenticeship Learning (AL), the reward function is unknown, and we denote the MDP without the reward function (also commonly called a controlled Markov chain) by  MDP\textbackslash R. Similarly, we denote a CMDP without a mapping of context to reward by \textbf{CMDP\textbackslash M}.
\\

\noindent
Instead of manually tweaking the reward to produce the desired behavior, the idea is to observe and mimic an expert. The literature on IRL is quite vast and dates back to \citep{ng2000algorithms, abbeel2004apprenticeship}. In this setting, the reward function (while unknown to the apprentice) is a linear combination of a set of known features as we defined above. The expert demonstrates a set of trajectories that are used to estimate the feature expectations of its policy $\pi_E$, denoted by $\mu_E $. The goal is to find a policy $\pi$, whose feature expectations are close to this estimate, and hence will have a similar return with respect to any weight vector $w$.\\

\noindent
Formally, AL is posed as a two-player zero-sum game, where the objective is to find a policy $\pi$ that does at least as well as the expert with respect to any reward function of the form $r(s) = w\cdot \phi(s), w\in\mathcal{W}$. That is we solve 
\begin{equation}
\label{eq:maxmin} 
   \max_{\pi \in \Pi}\min_{w\in \mathcal{W}}  \left[ w \cdot \mu(\pi) - w \cdot \mu_E \right]
\end{equation}
where $\Pi$ denotes the set of mixed policies \citep{abbeel2004apprenticeship}, in which a deterministic policy is sampled according to a distribution at time 0, and executed from that point on. Thus, this policy class can be represented as a convex set of vectors -- the distributions over the deterministic policies.\\

\noindent
They define the problem of approximately solving \cref{eq:maxmin} as AL, i.e., finding $\pi$ such that 
\begin{equation}
\label{eq:al}
\forall w\in \mathcal{W}: w\cdot \mu(\pi) \ge w\cdot \mu_E - \epsilon + f^\star.
\end{equation}

\noindent
If we denote the value of \cref{eq:maxmin}  by $f^\star$ then, due to the von-Neumann minimax theorem we also have that
\begin{equation}
    \label{eq:minmax} 
    f^\star = \min_{w\in \mathcal{W}} \max_{\pi\in \Pi} \left[ w \cdot \mu(\pi) - w \cdot \mu_E \right].
\end{equation}
We will later use this formulation to define the IRL objective, i.e., finding $w \in \mathcal{W}$ such that 
\begin{equation}
\label{eq:irl}
\forall \pi \in \Pi: w\cdot \mu_E \ge w\cdot \mu(\pi) - \epsilon - f^\star;
\end{equation}

\noindent
\citet{abbeel2004apprenticeship} suggested two algorithms to solve \cref{eq:al} for the case that $\mathcal{W}$ is a ball in the euclidean norm; one that is based on a maximum margin solver and a simpler projection algorithm. The latter starts with an arbitrary policy $\pi_0$ and computes its feature expectation $\mu_0$. At step $t$ they define a reward function using weight vector $w_t = \mu_E-\bar\mu_{t-1}$ and find the policy $\pi_t$ that maximizes it. $\Bar \mu_t$ is a convex combination of feature expectations of previous (deterministic) policies $\bar{\mu}_t = \sum^t _{j=1}\alpha_j \mu(\pi_j).$ They show that in order to get that $\norm{\bar{\mu}_T-\mu}\le \epsilon$, it suffices to run the algorithm for $T=O(\frac{k}{(1-\gamma)^2\epsilon^2}\log(\frac{k}{(1-\gamma)\epsilon}))$ iterations. 
\\

\noindent
Recently, \citet{zahavy2019apprenticeship} showed that the projection algorithm is in fact equivalent to a Frank-Wolfe method for finding the projection of the feature expectations of the expert on the feature expectations polytope – the convex hull of the feature expectations of all the deterministic policies in the MDP. The Frank-Wolfe analysis gives the projection method of \cite{abbeel2004apprenticeship} a slightly tighter bound of $T=O(\frac{k}{(1-\gamma)^2\epsilon^2}).$ In addition, a variation of the FW method that is based on taking “away steps” \citep{garber2016linearly,jaggi2013revisiting} achieves a linear rate of convergence, i.e., it is logarithmic in $\epsilon.$
\\

\noindent
Another type of algorithms, based on online convex optimization, was proposed by \citet{syed2008game}. In this approach, in each round the ``reward player" plays an online convex optimization algorithm on losses $l_t(w_t) = w_t\cdot (\mu_E - \mu(\pi_t))$; and the ``policy player" plays the best response, i.e, the policy $\pi_t$ that maximizes the return $\mu(\pi_t)\cdot w_t$ at time $t$. The results in \citep{syed2008game} use a specific instance of MDA where the optimization set is the simplex and distances are measured w.r.t $\norm{\cdot}_1.$ This version of MDA is known as multiplicative weights or Hedge.  The algorithm runs for $T$ steps and returns a mixed policy $\psi$ that draws with probability $1/T$ a policy $\pi_t, t=1,\ldots,T$. Thus, 
\begin{align}
f^\star &\le \frac{1}{T}\sum\nolimits_{t=1}^T\max_{\pi\in\Pi} \left[ w_t\cdot\mu(\pi)-w_t\cdot \mu_E\right] \nonumber \\
& =     \frac{1}{T}\sum\nolimits_{t=1}^T \left[ w_t\cdot\mu(\pi_t)-w_t\cdot\mu_E\right] \label{line:ineq}  \\
& \le   \min_{w\in\mathcal{W}}\frac{1}{T}\sum_{t=1}^T  w\cdot\left[ \mu(\pi_t)-\mu_E\right] +
{O}\left(\frac{\sqrt{\log(k)}}{(1-\gamma)\sqrt{T}}\right) \label{line:noregret}\\
& =  \min_{w\in\mathcal{W}} w \cdot \left( \mu(\psi) - \mu\right) + {O}\left(\frac{\sqrt{\log(k)}}{(1-\gamma)\sqrt{T}}\right),\label{eq:mwal}
\end{align}
where \cref{line:ineq} follows from the fact that the policy player plays the best response, that is, $\pi_t$ is the optimal policy w.r.t the reward  $w_t;$ \cref{line:noregret} follows from the fact that the reward player plays a no-regret algorithm, e.g., online MDA. Thus, they get that $\forall w\in \mathcal{W}: w\cdot \mu(\psi) \ge w\cdot \mu + f^\star - {O}\left(\frac{1}{\sqrt{T}}\right)$. \footnote{The $O$ notation hides the dependency in $k$ and $\gamma$ .} \\

\noindent
\textbf{Learned dynamics:} Finally, we note that majority of AL papers consider the problem of learning the transition kernel and initial state distribution as an orthogonal 'supervised learning' problem to the AL problem. That is, the algorithm start by approximating the dynamics form samples and then follows by executing the AL algorithm on the approximated dynamics \citep{abbeel2004apprenticeship,syed2008game}. In this paper we adapt this principle. We also note that it is possible to learn a transition kernel and an initial state distribution that are parametrized by the context. Existing methods, such as in \citet{modi2018markov}, can be used to learn contextual transition kernels. Furthermore, in domains that allow access to the real environment, \citet{10.1145/1102351.1102352} provides theoretical bounds for the estimated dynamics of the frequently visited state-action pairs. Thus, we assume $P_c$ is known when discussing suggested methods in \cref{sec:methods}, which enables the computation of feature expectations for any context and policy. In \cref{subsec:med_real} we present an example of this principle, where we use a context-dependent model to estimate the dynamics.

\vspace{-0.5cm}
\section{Methods}\label{sec:methods}

In the previous section we have seen AL algorithms for finding a policy that satisfies \cref{eq:al}. In a CMDP this policy will have to be a function of the context, but unfortunately, it is not clear how to analyze contextual policies. Instead, we follow the approach that was taken in the CMDP literature and aim to learn the linear mapping from contexts to rewards \citep{hallak2015contextual, modi2018markov, modi2019contextual}. This requires us to design an IRL algorithm instead of an AL algorithm, i.e., to solve \cref{eq:irl} rather than \cref{eq:al}. Concretely, the goal in Contextual IRL is to approximate the mapping $f^*(c)$ by observing an expert (for each context $c$, the expert provides a demonstration from $\pi^*_c$). \\


\noindent
This Section is organized as follows. We begin with \cref{subsec:mda}, where we formulate COIRL as a convex optimization problem and derive subgradient descent algorithms for it based on the Mirror Descent Algorithm (MDA). Furthermore, we show that MDA can learn efficiently even when there is only a single expert demonstration per context. This novel approach is designed for COIRL but can be applied to standard IRL problems as well.\\

\noindent
In \cref{subsec:ellips_method} we present a cutting plane method for COIRL that is based on the ellipsoid algorithm. This algorithm requires, in addition to demonstrations, that the expert evaluate the agent's policy and provide its demonstration only if the agent's policy is sub-optimal. \\

\noindent
In \cref{sec:related_work} we discuss how existing IRL algorithms can be adapted to the COIRL setting for domains with finite context spaces and how they compare to COIRL, which we later verify in the experiments section. Finally, in \cref{sec: transfer} we explore methods for efficient transfer to unseen contexts without additional planning. \\




\subsection{Mirrored Descent for COIRL}\label{sec:mda}
\label{subsec:mda}
\subsubsection{Problem formulation}
In this section, we derive and analyze convex optimization algorithms for COIRL that minimize the following loss function, 

\begin{equation}
\label{eq:loss}
     \text{Loss}(W) = \mathbb{E}_c \max_{\pi} \left[ f_W(c) \cdot \Big(\mu^\pi_c - \mu^*_c)\Big) \right] 
     = \mathbb{E}_c \Big[ f_W (c) \cdot \big(\mu^{\hat\pi_c(
     W)}_c - \mu^*_c\big)\Big] \, . 
\end{equation}

\begin{remark}
We analyze the descent methods for the linear mapping $f(c)=c^{T}W$. It is possible to extend the analysis to general function classes (parameterized by $W$), where $\frac{\partial f}{\partial W}$ is computable and $f$ is convex. In this case, $\frac{\partial f}{\partial W}$ aggregates to the descent direction instead of the context, $c$, and similar sample complexity bounds can be achieved. 
\end{remark}

\noindent
The following lemma suggests that if $W$ is a minimizer of \cref{eq:loss}, then the expert policy is optimal w.r.t reward $\hat R_c$  for any context.
\begin{lemma}
\label{lemma:loss}
$\text{Loss}(W)$ satisfies the following properties: \textbf{(1)} For any $W$ the loss is greater or equal to zero. \textbf{(2)} If $\text{Loss}(W)=0$ then for any context, the expert policy is the optimal policy w.r.t.~reward $\hat R_c(s)=c^TW\phi(s)$.
\end{lemma}

\noindent \textbf{Proof.}
We need to show that $\forall W$ , $\text{Loss}(W)\ge 0,$ and $\text{Loss}(W^*) = 0.$ Fix $W$. For any context $c,$ we have that $\mu^{\hat{\pi}_c(W)}_c$ is the optimal policy w.r.t.~reward $f_W (c),$ thus, $f_W (c) \cdot \big(\mu^{\hat\pi_c(W)}_c - \mu^*_c\big)\ge 0.$ Therefore we get that $\text{Loss}(W)\ge0.$ For $W^*,$ we have that $\mu^{\hat{\pi}_{c}(W)}_c = \mu^*_c,$ thus $\text{Loss}(W^*) = 0$. \\

\noindent
For the second statement , note that $\text{Loss}(W) = 0$ implies that $\forall c, \enspace f_W (c) \cdot \big(\mu^{\hat\pi_c(W)}_c - \mu^*_c\big)= 0.$ This can happen in one of two cases. (1) $\mu^{\hat\pi_c(W)}_c = \mu^*_c,$ in this case $\pi^*_{c},\hat{\pi}_{c}(W)$ have the same feature expectations. Therefore, they are equivalent in terms of value. (2) $\mu^{\hat{\pi}_{c}(W)}_c \neq \mu^*_c,$ but $f_W (c) \cdot \big(\mu^{\hat{\pi}_{c}(W)}_c - \mu^*_c\big)= 0.$ In this case, $\pi^*_{c},\hat{\pi}_{c}(W)$ have different feature expectations, but still achieve the same value w.r.t.~reward $f_W (c).$ Since $\hat{\pi}_{c}(W)$ is an optimal policy w.r.t.~this reward, so is $\pi^*_{c}.$ $\blacksquare$\\

\noindent
To evaluate the loss, the optimal policy $\hat \pi_c (W)$ and its features expectations $\mu^{\hat \pi_c (W)}_c$ must be computed for all contexts. Finding $\hat \pi_c (W)$, for a specific context, can be solved using standard RL methods, e.g., value or policy iteration. In addition, computing $\mu^{\hat \pi_c (W)}_c$ is equivalent to performing policy evaluation (solving a set of linear equations).\\

\noindent
However, since we need to use an algorithm (e.g. policy iteration) to solve for the optimal policy, \cref{eq:loss} is not differentiable w.r.t. $W$. We therefore consider two optimization schemes that do not involve differentiation: (i) subgradients and (ii) randomly perturbing the loss function (finite differences). Although the loss is non-differentiable,  \cref{thm:covexloss} below shows that in the special case that $f_W(c)$ is a linear function, \cref{eq:loss} is convex and Lipschitz continuous. Furthermore, it provides a method to compute its subgradients. 

\begin{lemma}
\label{thm:covexloss}
    Let $f_W(c) = c^TW$ such that $\text{Loss}(W),$ denoted by $ L_{\text{lin}}(W)$,  is given by 
    \begin{equation*}
     L_{\text{lin}}(W) = \mathbb{E}_c \Big[ c^T W \cdot \big(\mu^{\hat{\pi}_{c}(W)}_c - \mu^*_c\big)\Big] \enspace.
    \end{equation*}
    We have that:
    \begin{enumerate}
        \item $L_{\text{lin}}(W)$ is a convex function.
        \item $g(W) = \mathbb{E}_c \left[c \odot \big(\mu^{\hat{\pi}_{c}(W)}_c - \mu^*_c\big) \right]$ is a subgradient of $ L_{\text{lin}}(W)$.
        \item $L_{\text{lin}}$ is a Lipschitz continuous function, with Lipschitz constant $L = \frac{2}{1-\gamma}$ w.r.t.~$\norm{\cdot}_\infty$ and $L = \frac{2\sqrt{dk}}{1-\gamma}$ w.r.t.~$\norm{\cdot}_2$. 
    \end{enumerate}
\end{lemma}
\noindent

\noindent In the supplementary material we provide the proof for the Lemma. The proof follows the definitions of convexity and subgradients, using the fact that for each $W$ we compute the optimal policy for reward $c^TW$. The Lipschitz continuity of $L_\text{Lin}(W)$ is related to the simulation lemma \citep{kearns2002near}, that is, a small change in the reward results in a small change in the optimal value.\\

\noindent
Note that $g(W)\in\mathbb{R}^{d\times k}$ is a matrix; we will sometimes refer to it as a matrix and sometimes as a flattened vector, depending on the context. Finally, $g(W)$ is given in expectation over contexts, and in expectation over trajectories (feature expectations). We will later see how to replace $g(W)$ with an unbiased estimate, which can be computed by aggregating state features from a single expert trajectory sample. 

\subsubsection{Algorithms}
\cref{thm:covexloss} identifies $L_\text{Lin}(W)$ as a convex function and provides a method to compute its subgradients. A standard method for minimizing a convex function over a convex set is the subgradient projection algorithm \citep{bertsekas1997nonlinear}. The algorithm is is given by the following iterates:
\begin{equation*}
    W_{t+1} = \text{Proj}_{\mathcal{W}} \{ W_t - \alpha_t g(W_t)\} \enspace ,
\end{equation*}
where $f(W_t)$ is a convex function,  $g(W_t)$ is a subgradient of $f(W_t)$, and $\alpha_t$ the learning rate. $\mathcal{W}$ is required to be a convex set; we will consider two particular cases, the $\ell_2$ ball \citep{abbeel2004apprenticeship} and the simplex \citep{syed2008game}\footnote{Scaling of the reward by a constant does not affect the resulting policy, thus, these sets are not restricting.}. 

\noindent Next, we consider a generalization of the subgradient projection algorithm that is called the mirror descent algorithm \cite[MDA]{nemirovsky1983problem}:
\begin{equation}
    \label{eq:mda_iter}
     W_{t+1} = \arg\min_{W\in\mathcal{W}} \left\{ 
     W\cdot \nabla_f(W_t) + \frac{1}{\alpha_t}D_\psi(W,W_t)
     \right\} \enspace ,
\end{equation}
where $D_\psi(W,W_t)$ is a Bregman distance\footnote{We refer the reader to the supplementary material for definitions of the Bregman distance, the dual norm, etc.}, associated with a strongly convex function $\psi$. The following theorem characterizes the convergence rate of MDA.

\begin{theorem}[Convergence rate of MDA]\label{thm:mda}
Let $\psi$ be a $\sigma$-strongly convex function on $\mathcal{W}$ w.r.t.~$\norm{\cdot}$, and let $D^2 = \sup_{W_1,W_2\in\mathcal{W}}D_\psi (W
_1,W_2)$. Let $f$ be convex and $L$-Lipschitz continuous w.r.t.~$\norm{\cdot}$. Then, MDA with $\alpha_t = \frac{D}{L}\sqrt{\frac{2\sigma}{t}}$ satisfies:
\begin{equation*}
    f\left(\frac{1}{T}\sum_{t=1}^T W_t\right) - f(W^*) \le DL\sqrt{\frac{2}{\sigma T}} \enspace .
\end{equation*}
\end{theorem}
We refer the reader to \citet{beck2003mirror} and \citet{bubeck2015convex} for the proof. Specific instances of MDA require one to choose a norm and to define the function $\psi.$  Once those are defined, one can compute $\sigma, D$ and $L$ which define the learning rate schedule. Below, we provide two MDA instances (see, for example \citet{beck2003mirror} for derivation) and analyze them for COIRL.\\

\noindent
\textbf{Projected subgradient descent (PSGD):} Let $\mathcal{W}$ be an $\ell_2$ ball with radius $1$. Fix $||\cdot||_2$, and $\psi(W) = \frac{1}{2}||W||_2^2.$ $\psi$ is strongly convex w.r.t.~$||\cdot||_2$ with $\sigma=1.$ The associated Bregman divergence is given by $D_\psi(W_1, W_2) = 0.5 ||W_1 - W_2||_2^2.$ Thus, mirror descent is equivalent to PSGD. $D^2 = \max_{W_1,W_2 \in \mathcal{W}}D_\psi(W_1, W_2) \le 1,$ and according to \cref{thm:covexloss}, $L = \frac{2\sqrt{dk}}{1-\gamma}$.
Thus, we have that the learning rate is $\alpha_t = (1-\gamma)\sqrt{\frac{1}{2dkt}}$ and the update to $W$ is given by $$\tilde W =  W_t - \alpha_t g_t, \enspace W_{t+1} = \tilde W/||\tilde W||_2,$$
and according to \cref{thm:mda} we have that after $T$ iterations,
\begin{equation*}
    L_{\text{lin}}\left(\frac{1}{T}\sum\nolimits_{t=1}^T W_t\right) - L_{\text{lin}}(W^*) \le  \mathcal{O}\left( \frac{\sqrt{dk}}{(1-\gamma)\sqrt{T}} \right) \, .
\end{equation*}

\noindent
\textbf{Exponential Weights (EW):} Let $\mathcal{W}$ be the standard ${dk-1}$ dimensional simplex. Let ${\psi (W) = \sum _i W(i)\log(W(i))}$. $\psi$ is strongly convex w.r.t.~$||\cdot||_1$ with ${\sigma=1}$. We get that the associated Bregman divergence is given by $$D_\psi(W_1, W_2) = \sum_i W_1(i) \log(\frac{W_1(i)}{W_2(i)}),$$ also known as the Kullback-Leibler divergence. In addition, $$D^2 = \max_{W_1,W_2 \in \mathcal{W}}D_\psi(W_1, W_2) \le \log(dk)$$ and according to \cref{thm:covexloss}, $L = \frac{2}{1-\gamma}$. Thus, we have that the learning rate is $\alpha_t = (1-\gamma)\sqrt{\frac{\log(dk)}{2t}}.$ Furthermore, the projection onto the simplex w.r.t.~to this distance amounts to a simple renormalization $W \leftarrow W/||W||_1$.  Thus, we get that MDA is equivalent to the exponential weights algorithm and the update to $w$ is given by $$
\forall i \in [1..dk], \enspace \tilde W(i) = W_{t}(i) \exp\left(-\alpha_t g_t (i)\right), \enspace W_{t+1} = \tilde W/||\tilde W||_1.$$
Finally, according to \cref{thm:mda} we have that after $T$ iterations,

\begin{equation*}
    L_{\text{lin}}\left(\frac{1}{T}\sum\nolimits_{t=1}^T W_t\right) - L_{\text{lin}}(W^*) \le  \mathcal{O}\left( \frac{\sqrt{\log(dk)}}{(1-\gamma)\sqrt{T}} \right) \, .
\end{equation*}

\begin{algorithm}[H]
   \caption{MDA for COIRL}
    \label{alg:convexIRL}
\begin{algorithmic}
    \STATE \textbf{Input:} a norm  $||\cdot||$, a strongly convex function $\psi$ with strong convexity parameter $\sigma$ w.r.t to the norm, a convex set $\mathcal{W}$ with diameter $D$ w.r.t the norm, $L$ a Lipschitz constant,  $T$ number of iterations
    \STATE \textbf{Initialize} $W_1 \in \mathcal{W}$
    \FOR{$t = 1,\ldots,T$}
        \STATE Observe $c, \mu^*_c$
        \STATE Compute $\hat \pi_{c}(W_t),\mu^{\hat \pi_{c}(W_t)}_c$ 
        \STATE Compute a subgradient $g_t = c \odot \big(\mu^{\hat{\pi}_{c}(W_t)}_c - \mu^*_c\big) $
        \STATE Set learning rate $\alpha_t = \frac{D}{L}\sqrt{\frac{2\sigma}{t}}$
        \STATE Update:  $W_{t+1} = \arg\min_{W\in\mathcal{W}} \left\{ 
     W\cdot g_t + \frac{1}{\alpha_t}D_\psi(W,W_t)
     \right\}$
    \ENDFOR
    \STATE \textbf{return} $\frac{1}{t}\sum_{t=1}^T W_t$
\end{algorithmic}
\end{algorithm}

\noindent
\textbf{Evolution strategies for COIRL.} Next, we consider a derivative-free algorithm for computing subgradients, based on \emph{Evolution Strategies} \cite[ES]{salimans2017evolution}. For convex optimization problems, ES is a gradient-free descent method based on computing finite differences \citep{nesterov2017random}. The subgradient in ES is computed by sampling $m$ random perturbations and computing the loss for them, in the following form

\begin{align*}
    \label{eq:es_g}
    \text{For } j =  &1, ..., m \text{ do} \\
    &\text{Sample }  u_j \sim \mathcal{N}(0,\rho^{2}) \in\mathcal{R}^{dk},\\
    &g^j = \text{Loss} \left(W_t + \frac{\nu u_j}{||u_j||} \right) \frac{ \nu u_j}{||u_j||} ,  \\ 
    \text{End For}&,
\end{align*}
and the subgradient is given by 
\begin{equation}
    \label{eq:es_g}
    g_t = \frac{1}{m\rho} \sum_{j=1}^m g^j.
\end{equation}
\cref{thm: es} presents the sample complexity of PSGD with the subgradient in \cref{eq:es_g} for the case that the loss is convex, as in $L_\text{Lin}$. 
While this method has looser upper-bound guarantees compared to MDA (\cref{thm:mda}), \citet{nesterov2017random} observed that in practice, it often outperforms subgradient-based methods. Thus, we test ES empirically and compare it with the subgradient method (\cref{sec:mda}). Additionally, \citet{salimans2017evolution} have shown the ability of ES to cope with high dimensional non-convex tasks (DNNs).


\begin{theorem}[ES Convergence Rate \citep{nesterov2017random}]\label{thm: es}
Let $L_\text{lin}(W)$ be a non-smooth convex function with Lipschitz constant $L$, such that $||W_0 - W^*|| \leq D$, step size of $\alpha_t = \frac{D}{(dk+4)\sqrt{T+1}L}$ and $\nu \leq \frac{\epsilon}{2L\sqrt{dk}}$ then in $T = \frac{4(dk+4)^2D^2 L^2}{\epsilon^2}$ ES finds a solution which is bounded by $\mathbb{E}_{U_{T-1}} [L_\text{lin}(\hat W_T)] - L_\text{lin}(W^*) \leq \epsilon$, where ${U_T = \{ u_0, \ldots, u_T \}}$ denotes the random variables of the algorithm up to time $T$ and $\hat W_T = \argmin_{t = 1, \ldots, T} L_\text{lin}(W_t)$.
\end{theorem}

\noindent
\textbf{Practical MDA:}
One of the ``miracles" of MDA is its robustness to noise. If we replace $g_t$ with an unbiased estimate $\tilde g_t,$ such that $\mathbb{E} \tilde{g}_t = g_t$  and $\mathbb{E} \norm{\tilde{g}_t} \le L$, we obtain the same convergence results as in \cref{thm:mda} \citep{robbins1951stochastic} (see, for example, \citet[Theorem 6.1]{bubeck2015convex}). Such an unbiased estimate can be obtained in the following manner: (i) sample a context $c_t$, (ii) compute $\mu^{\pi^*_{c_t}(W_t)}_{c_t}$, (iii) observe a single expert demonstration $\tau^E_i = \{s_0^i,a_0,s_1^i,a_1,\ldots,\},$ where $a_i$ is chosen by the expert policy  $\pi^*_{c_t}$ (iv) let $\hat \mu_i = \sum _{t\in[0,\ldots,|\tau^E_i|-1]}\gamma^t \phi(s_t^i)$ be the accumulated discounted features across the trajectory such that $\mathbb{E} \hat \mu_i = \mu^*_{c_t}$.\\

\noindent
However, for $\hat \mu_i$ to be an unbiased estimate of $\mu^*_{c_t}$, $\tau^E_i$ needs to be of infinite length.
Thus one can either (1) execute the expert trajectory online, and terminate it at each time step with probability $1-\gamma$ \citep{kakade2002approximately}, or (2) execute a trajectory of length $H=\frac{1}{1-\gamma}\log(1/\epsilon_H)$. The issue with the first approach is that since the trajectory length is unbounded, the estimate $\hat \mu _i $ cannot be shown to concentrate to  $\mu^*_{c_t}$ via Hoeffding type inequalities. Nevertheless, it is possible to obtain a concentration inequality using the fact that the length of each trajectory is bounded in high probability (similar to \citet{zahavy2019average}).
The second approach can only guarantee that $\norm{g_t - \mathbb{E}\tilde g_t}\le \epsilon_H$ \citep{syed2008game}. Hence, using the robustness of MDA to adversarial noise \citep{zinkevich2003online}, we get that MDA converges with an additional error of $\epsilon_H$, i.e., 
\begin{equation*}
    L_{\text{lin}}\left(\frac{1}{T}\sum_{t=1}^T W_t\right) - L_{\text{lin}}(W^*) \le  \mathcal{O}\left(\frac{1}{\sqrt{T}}\right) + \epsilon_H \enspace .
\end{equation*}    
While this sampling mechanism has the cost of a controlled bias, usually it is more practical, in particular, if the trajectories are given as a set of demonstrations (offline data).

\subsection{Ellipsoid algorithms for COIRL}\label{subsec:ellips_method}

\begin{wrapfigure}{r}{0.3\textwidth}
\vspace{-0.7cm}
\includegraphics[width=\linewidth]{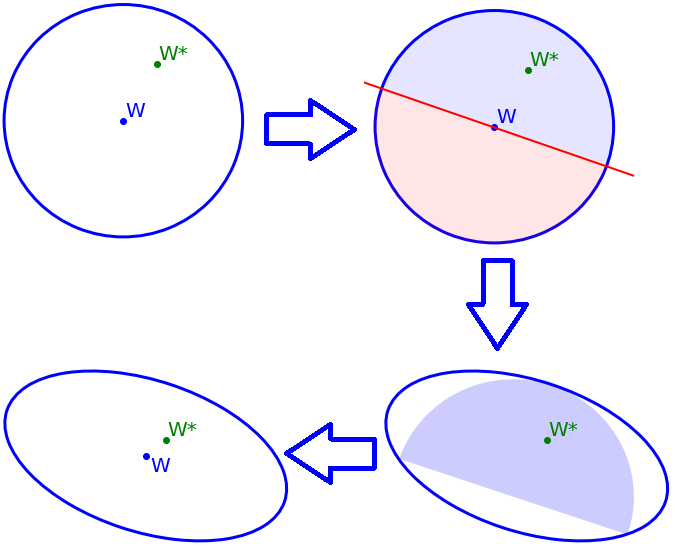}
\vspace{-0.5cm}
\caption{The ellipsoid algorithm proceeds in an iterative way, using linear constraints to gradually reduce the size of the ellipsoid until the center defines an $\epsilon$-optimal solution.}
\label{fig:img_ellipsoid}
\end{wrapfigure} 
In this section we present the ellipsoid method, introduced to the IRL setting by \citet{amin2017repeated}. We extend this method to the contextual setting, and focus on finding a linear mapping $W \in \mathcal{W} $ where  $\mathcal{W} = \{W: ||W||_\infty \leq 1\}$, and $W^*\in\mathcal{W}$. The algorithm, illustrated in \cref{fig:img_ellipsoid}, maintains an ellipsoid-shaped feasibility set for $W^*$. In each iteration, the algorithm receives a demonstration which is used to create a linear constraint, halving the feasibility set. The remaining half-ellipsoid, still containing $W^*$, is then encapsulated by a new ellipsoid. With every iteration, this feasibility set is reduced until it converges to $W^*$.\\

\noindent
Formally, an ellipsoid is defined by its center -- a vector $u$, and by an invertible matrix $Q$: $\{x:(x-u)Q^{-1}(x-u) \le 1\}$. The feasibility set for $W^*$ is initialized to be the minimal sphere containing $\{W: ||W||_\infty \leq 1\}$. At every step $t$, the current estimation $W_t$ of $W^*$ is defined as the center of the feasibility set, and the agent acts optimally w.r.t. the reward function $\hat R_c(s) = c^TW_t\phi(s)$. If the agent performs sub-optimally, the expert provides a demonstration in the form of its feature expectations for $c_t$:  $\mu^*_{c_t}$. These feature expectations are used to generate a linear constraint (hyperplane) on the ellipsoid that is crossing its center. Under this constraint, we construct a new feasibility set that is half of the previous ellipsoid, and still contains $W^*$. For the algorithm to proceed, we compute a new ellipsoid that is the minimum volume enclosing ellipsoid (MVEE) around this "half-ellipsoid". These updates are guaranteed to gradually reduce the volume of the ellipsoid, as shown in \cref{lemma: ellipsoid},  until its center is a mapping which induces $\epsilon$-optimal policies for all contexts.

\begin{lemma}[\cite{boyd1991linear}]\label{lemma: ellipsoid}
If $B \subseteq \mathbb{R}^D$ is an ellipsoid with center $w$, and $x\in\mathbb{R}^D \backslash \{0\}$, we define $B^+ = \text{MVEE}(\{ \theta\in B: (\theta-w)^Tx \geq 0 \})$, then: $\frac{Vol(B^+)}{Vol(B)} \leq e^{-\frac{1}{2(D+1)}} \enspace .$
\end{lemma}

\noindent
\cref{thm: upper bound} below shows that this algorithm achieves a polynomial upper bound on the number of sub-optimal time-steps. The proof, found in \cref{sec:ellipsoid_proofs}, is adapted from \citep{amin2017repeated} to the contextual setup.\\

\vspace{-0.2cm}
\begin{algorithm}[h]
\caption{Ellipsoid algorithm for COIRL}\label{alg:ellipsoid}
\begin{algorithmic}
\STATE \textbf{Initialize:} $\Theta_0 \leftarrow B_\infty(0,1)=$\scalebox{0.9}{$\{x\in\mathbb{R}^{d\cdot k}: ||x||_\infty \leq 1 \}$} 
\STATE $\Theta_1 \leftarrow \text{MVEE} (\Theta_0): \underline{W}_1 = 0, Q_1 = dkI$
\FOR{$t=1,2,\ldots$}
    \STATE Observe $c_t,$ let $\underline{W}_t$ be the center of $\Theta_t$
    \STATE Play episode using $\hat \pi_t = \argmax_{\pi} c_t^T W_t \mu^\pi_{c_t}$
    \IF{$V^*_{c_t} - V^{\hat\pi_t}_{c_t} > \epsilon$}
        \STATE $\mu^*_{c_t}$ is revealed
        \STATE Let $a_t = c_t\odot\left(\mu^*_{c_t}-\mu^{\hat \pi_t}_{c_t}\right)$
        \STATE $\Theta_{t+1}$ $\leftarrow \text{MVEE} \left(\left\{
        \theta\in\Theta_t:\theta^T a_t \geq \underline{W}_t^T a_t  \right\}\right)$
    \ELSE
        \STATE $\Theta_{t+1} \leftarrow \Theta_t$
    \ENDIF
\ENDFOR
\STATE
\STATE \textbf{MVEE($\left\{\theta\in\Theta_t:\theta^T a_t \geq \underline{W}_t^T a_t  \right\}$):}
\STATE \quad $\tilde{a}_t = \frac{-1}{\sqrt{a_t^TQ_ta_t}}a_t$
\STATE \quad $\underline{W}_{t+1} = \underline{W}_t - \frac{1}{dk+1}Q_t\tilde{a}_t $
\STATE \quad $Q_{t+1} = \frac{d^2k^2}{d^2k^2-1}(Q_t - \frac{2}{dk+1}Q_t\tilde{a}_t\tilde{a}_t^TQ_t)$
\end{algorithmic}
\end{algorithm}

\begin{theorem}\label{thm: upper bound}
In the linear setting where $R^*_c(s) = c^TW^*\phi(s)$, for an agent acting according to Algorithm 1, the number of rounds in which the agent is not $\epsilon$-optimal is $\mathcal{O}(d^2k^2\log(\frac{d k}{(1-\gamma)\epsilon}))$.
\end{theorem}
\begin{remark}
    Note that the ellipsoid method presents a new learning framework, where demonstrations are only provided when the agent performs sub-optimally. Thus, the theoretical results in this section cannot be directly compared with those of the descent methods. We further discuss this in \cref{section:ellipsoid_experiments}.
\end{remark}

\begin{remark}
    The ellipsoid method does not require a distribution over contexts - an adversary may choose them. MDA can also be easily extended to the adversarial setting via known regret bounds on online MDA  \citep{hazan2016introduction}.
\end{remark}

\noindent
\textbf{Practical ellipsoid algorithm:}\label{sec:ellipsoid_traj}
In real-world scenarios, it may be impossible for the expert to evaluate the value of the agent's policy, i.e. check if $V^*_{c_t} - V^{\hat\pi_t}_{c_t} > \epsilon$, and to provide its policy or feature expectations $\mu^*_{c_t}$. To address these issues, we follow \citet{amin2017repeated} and consider a relaxed approach, in which the expert evaluates each of the individual actions performed by the agent rather than its policy (\cref{alg:trajectories}). When a sub-optimal action is chosen, the expert provides finite roll-outs instead of its policy or feature expectations. We define the expert criterion for providing a demonstration to be $Q^*_{c_t}(s,a) + \epsilon < V^*_{c_t}(s)$ for each state-action pair $(s,a)$ in the agent's trajectory.\\

\noindent
\textit{Near-optimal experts:} In addition, we relax the optimality requirement of the expert and instead assume that, for each context $c_t$, the expert acts optimally w.r.t. $W^*_t$ which is close to $W^*$; the expert also evaluates the agent w.r.t. this mapping. This allows the agent to learn from different experts, and from non-stationary experts whose judgment and performance slightly vary over time. If a sub-optimal action w.r.t. $W^*_t$ is played at state $s$, the expert provides a roll-out of $H$ steps from $s$ to the agent. As this roll-out is a sample of the optimal policy w.r.t. $W^*_t$, we aggregate $n$ examples to assure that with high probability, the linear constraint that we use in the ellipsoid algorithm does not exclude $W^*$ from the feasibility set. Note that these batches may be constructed across different contexts, different experts, and different states from which the demonstrations start. \cref{thm: trajectories}, proven in \cref{sec:ellipsoid_proofs}, upper bounds the number of sub-optimal actions that \cref{alg:trajectories} chooses.\footnote{MDA also works with near optimal experts due to the robustness of MDA. The analysis of this case is identical to the analysis of biased trajectories, as we discuss in the end of \cref{sec:mda}.}

\begin{algorithm}[H]
\caption{Batch ellipsoid algorithm for COIRL}\label{alg:trajectories}
\begin{algorithmic}
\STATE \textbf{Initialize:} $\Theta_0 \leftarrow B_\infty(0,1)=\{x\in\mathbb{R}^{d\cdot k}: ||x||_\infty \leq 1 \}$ 
\STATE $\Theta_1 \leftarrow \text{MVEE} (\Theta_0)$
\STATE $i \leftarrow 0, \Bar{Z} \leftarrow 0, \Bar{Z}^* \leftarrow 0$
\FOR{$t=1,2,3,...$}
    \STATE $c_t$ is revealed, Let $\underline{W}_t$ be the center of $\Theta_t$
    \STATE Play episode using $\hat \pi_t = \argmax_{\pi} c_t^TW_t\mu^\pi_{c_t}$
    \STATE $\Theta_{t+1} \leftarrow \Theta_t$
    \IF{a sub-optimal action $a$ is played at state $s$}
        \STATE Expert provides H-step trajectory $(s^E_0=s,s^E_1,...,s^E_H)$. Let $\hat{x}_i^{*,H}$ be the H-step sample of the expert's feature expectations for $\xi_i'=\mathds{1}_s$: $\hat{x}_i^{*,H}=\sum_{h=0}^{H} \gamma^h \phi(s^E_h)$
        \STATE Let $x_i$ be the agent's feature expectations for $\xi_i':$ 
        $E_{\xi_i',P_{c_t},\pi_t}[\sum_{h=0}^{\infty} \gamma^h \phi(s_h)]$
        \STATE Denote $z_i = c_t\odot x_i$, $\hat{z}_i^{*,H} = c_t\odot\hat{x}_i^{*,H}$
        \STATE $i \leftarrow i+1, \Bar{Z} \leftarrow \Bar{Z}+ z_i, \Bar{Z}^* \leftarrow \Bar{Z}^*+\hat{z}_i^{*,H}$
        \IF{$i=n$}
            \STATE $\Theta_{t+1}$ $\leftarrow \text{MVEE} \bigg(\bigg\{        \theta\in\Theta_t:\Big(\theta-\underline{W}_t\Big)^T \cdot (\frac{\Bar{Z}^*}{n}-\frac{\Bar{Z}}{n}) \geq 0 \bigg\}\bigg)$
            \STATE $i \leftarrow 0, \Bar{Z} \leftarrow 0, \Bar{Z}^* \leftarrow 0$
        \ENDIF
    \ENDIF
\ENDFOR
\end{algorithmic}
\end{algorithm}

\begin{theorem}\label{thm: trajectories}
For an agent acting according to \cref{alg:trajectories}, $H=\ceil{\frac{1}{1-\gamma}\log(\frac{8k}{(1-\gamma)\epsilon})}$ and $n=\ceil{\frac{512k^2}{(1-\gamma)^2\epsilon^2}\log(4dk(dk+1)\log(\frac{16k\sqrt{dk}}{(1-\gamma)\epsilon})/\delta)}$, with probability of at least $1-\delta$, if $\forall t: \underline{W}^*_t \in B_\infty(\underline{W}^*,\frac{(1-\gamma)\epsilon}{8k}) \cap \Theta_0$ the number of rounds in which a sub-optimal action is played is $\mathcal{O}\Big(\frac{d^2k^4}{(1-\gamma)^2\epsilon^2}\log\big(\frac{dk}{(1-\gamma)\delta\epsilon}\log(\frac{dk}{(1-\gamma)\epsilon})\big)\Big)\,$.
\end{theorem}

\noindent
The theoretical guarantees of the algorithms presented so far are summarized in \cref{table:discussion}. We can see that MDA, in particular EW, achieves the best scalability. In the unrealistic case where the expert can provide its feature expectations, the ellipsoid method has the lowest sample complexity. However, in the realistic scenario where only samples are provided, the sample complexity is identical across all methods. We also note that unlike MDA and ES, it isn't possible to extend the ellipsoid method to work with DNNs. Overall, the theoretical guarantees favor the MDA methods when it comes to the realistic setting.
\begin{table}[h]
    \setlength\extrarowheight{8pt}
\caption{Summary of theoretical guarantees.}
\begin{center}
\resizebox{\textwidth}{!}{\begin{tabular}{|cc|c|c|c|c|c|}
    \hline
    \\[-1.9em]
    & & \multicolumn{2}{c|}{\textbf{Scalability}} & \multicolumn{2}{c|}{\textbf{Sample Complexity}} & \textbf{\thead{\scriptsize Extension \\to DNNs}} \\ \hline
    & & Feature expectations & Sampled trajectory & Feature expectations & Sampled trajectory & \\ \hline\hline
    \multicolumn{1}{|c|}{\multirow{2}{*}{MDA}} & PSGD & \multicolumn{2}{c|}{$\mathcal{O}(dk)$} & \multirow{3}{*}{$\mathcal{O}\left(\frac{1}{\epsilon^2}\right)$} & \multirow{4}{*}{$\mathcal{O}\left(\frac{1}{\epsilon^2}\right)$} & \cmark \\ \cline{2-4}\cline{7-7}
    \multicolumn{1}{|c|}{} & EW & \multicolumn{2}{c|}{$\mathcal{O}(\log{dk})$} & & & \xmark \\ \cline{1-4}\cline{7-7}
    \multicolumn{2}{|c|}{ES} & $\mathcal{O}(dk)$ & $\mathcal{O}(d^2k^2)$ &  &  & \cmark \\ \cline{0-4}\cline{7-7}
    \multicolumn{2}{|c|}{Ellipsoid} & $\mathcal{O}(d^2k^2)$ & $\mathcal{O}(d^2k^4)$ & $\mathcal{O}\left(\log\frac{1}{\epsilon}\right)$ &  & \xmark \\ \hline
\end{tabular}}
\end{center}
\label{table:discussion}
\end{table}
\vspace{-1cm}
\subsection{Existing approaches}\label{sec:related_work}
We focus our comparisons to methods that can be used for zero-shot generalization across contexts or tasks. Hence, we omit discussion of ``meta inverse reinforcement learning" methods which focus on few-shot generalization \citep{xu2018learning}. Our focus is on two approaches: (1) standard IRL methods applied to a model which incorporates the context as part of the state, and (2) contextual policies through behavioral cloning (BC) \citep{pomerleau1989alvinn}.

\subsubsection{Application of IRL to COIRL problems}\label{subsec:largemdp}
We first examine the straight-forward approach of incorporating the contextual information into the state, i.e., defining $\mathcal{S}' = \mathcal{C} \times \mathcal{S}$, and applying standard IRL methods to one environment which captures all contexts. This construction limits the context space to a finite one, as opposed to COIRL which works trivially with an infinite number of contexts. At first glance, this method results in the same scalability and sample complexity as COIRL; however, when considering the inner loop in which an optimal policy is calculated, COIRL has the advantage of a smaller state space by a factor of $|\mathcal{C}|$. This results in significantly better run-time when considering large context spaces. In \cref{subsec:exp_irl}, we present experiments that evaluate the run-time of this approach, compared to COIRL, for increasingly large context spaces. These results demonstrate that the run-time of IRL scales with $|\mathcal{C}|$ while the run-time of COIRL is unaffected by $|\mathcal{C}|$, making COIRL much more practical for environments with many or infinite contexts.

\subsubsection{Contextual policies}\label{subsec:contextual_poli}
Another possible approach is to use Behavioral Cloning (BC) to learn contextual policies, i.e., policies that are functions of both state and context $\pi(c,s)$. In BC, the policy is learned using supervised learning methods, skipping the step of learning the reward function. While BC is an intuitive method, with successful applications in various domains \citep{bojarski2016end,ratliff2007imitation}, it has a fundamental flaw; BC violates the i.i.d. assumptions of supervised learning methods, as the learned policy affects the distribution of states it encounters. This results in a covariate shift in test-time leading to compounding errors \citep{ross2010efficient,ross2011reduction}. Methods presented in \citet{ross2011reduction,laskey2017iterative} mitigate this issue but operate outside of the offline framework. This explains why BC compares unfavorably to IRL methods, especially with a limited number of available demonstrations \citep{ho2016generative,ghasemipour2019smile}. In \cref{subsubsec:exp_bc}, we provide experimental results that exhibit the same trend. These results demonstrate how matching actions on the train set poorly translates to value on the test set, until much of the expert policy is observed. While a single trajectory per context suffices for COIRL, BC requires more information to avoid encountering unfamiliar states. We also provide a hardness result for learning a contextual policy for a linear separator hypothesis class, further demonstrating the challenges of this approach.

\subsection{Transfer across contexts in test-time}\label{sec: transfer}
In this section, we examine the application of the learned mapping $W$ when encountering a new, unseen context in test-time. Unlike during training, in test-time the available resources and latency requirements may render re-solving the MDP for every new context infeasible. We address this issue by leveraging optimal policies $\{\pi^*_{c_j}\}_{j=1}^{N}$ for contexts $\{c_j\}_{j=1}^{N}$ which were previously calculated during training or test time. We separately handle context-independent dynamics and contextual dynamics by utilizing (1) generalized policy improvement (GPI) \citep{barreto2017successor}, and (2) the simulation lemma \citep{kearns2002near}, respectively.\\

\noindent
For context-independent dynamics, the framework of \citet{barreto2017successor} can be applied to efficiently transfer knowledge from previously observed contexts $\{c_j\}_{j=1}^{N}$ to a new context $c$. As the policies $\{\pi^*_{c_j}\}_{j=1}^{N}$ were computed, so were their feature expectations, starting from any state. As the dynamics are context-independent, these feature expectations are also valid for $c$, enabling fast computation of the corresponding Q-functions, thanks to the linear decomposition of the reward. GPI generalizes policy improvement, allowing us to use these Q-functions to create a new policy that is as good as any of them and potentially strictly better than them all. The following theorem, a parallel of Theorem 2 in \citet{barreto2017successor}, defines the GPI calculation and provides the lower bound on its value. While these theorems and their proofs are written for $W^*$, the results hold for any $W \in \mathcal{W}$.
\begin{theorem}[\citet{barreto2017successor}]\label{thm: gpi}
Let $\phi_{max} = \max_s ||W^*\phi(s)||_1$, $\{c_j\}_{j=1}^N \subseteq \mathcal{C}$, $c \in \mathcal{C}$, and ${\pi(s) \in \argmax_a \max_j Q^{\pi^*_{c_j}}_{c}(s,a)}$. If the dynamics are context independent, then:
\begin{equation*}
    V^*_{c} - V^\pi_{c} \le 2 \frac{\phi_{max}}{1-\gamma}\min_j||c-c_j||_\infty \enspace .
\end{equation*}
\end{theorem}

\noindent
When the dynamics are a function of the context, the feature expectations calculated for $\{c_j\}_{j=1}^N$ are not valid for $c$, thus GPI can not be used efficiently. However, due to the linearity and therefore continuity of the mapping, similar contexts induce similar environments. Thus, it is intuitive that if we know the optimal policy for a context, it should transfer well to nearby contexts without additional planning. This intuition is formalized in the simulation lemma, which is used to provide bounds on the performance of a transferred policy in the following theorem. 
\begin{theorem}\label{thm: transfer}
Let $c,c_j\in \mathcal{C}, \phi_{max} = \max_s ||W^*\phi(s)||_1$, $V_{max} = \max_{c,s} |V^*_c(s)| $. Then:
\begin{equation*}
    V^*_{c} - V^{\pi^*_{c_j}}_{c} \le 2 \frac{\phi_{max} + \gamma d V_{max}}{\gamma(1-\gamma)}||c-c_j||_\infty \enspace .
\end{equation*}
\end{theorem}
\begin{remark}
The bound depends on $\mathcal{W}$. For example, for $\mathcal{W}=\Delta_{dk-1}$, the bound is $2 \frac{1 - \gamma + \gamma d}{\gamma(1-\gamma)^2}||c-c_j||_\infty$, and for $\mathcal{W}=B_\infty(0,1)$ the bound is $\frac{2dk}{\gamma(1-\gamma)^2}||c-c_j||_\infty$.
\end{remark}
\begin{remark}
If the dynamics are independent of the context, the term $\gamma d V_{max}$ is omitted from the bound.
\end{remark}

\noindent
Using these methods, one can efficiently find a good policy for a new context $c$, either as a good starting point for policy/value iteration which will converge faster or as the final policy to be used in test-time. The last thing to consider is the construction of the set $\{c_j\}_{j=1}^{N}$. As COIRL requires computing the optimal policies for $W$ during training, the training contexts are a natural set to use. In addition, as suggested in \citet{barreto2017successor}, we may reduce this set or enhance it in a way that maintains a covering radius in $\mathcal{C}$ and guarantees a desired level of performance. If the above methods are used as initializations for calculating the optimal policy, the set can be updated in test-time as well.

\section{Experiments}\label{sec:experiments}
In the previous sections we described the theoretical COIRL problem, proposed methods to solve it and analyzed them. In this section our goal is to take COIRL from theory to practice. This section presents the process and the guidelines we follow to achieve this goal in a step-by-step manner, to bridge the gap between theoretical and real-life problems through a set of experiments\footnote{The code used in these experiments is provided in following repository \href{https://github.com/coirl/coirl_code}{github.com/coirl/coirl\_code}  .}. \\

\noindent
We begin by focusing on the \emph{grid world} and \emph{autonomous driving simulation} environments. As these are relatively small domains, for which we can easily compute the optimal policy, they provide easily accessible insight into the behavior of each method and allow us to eliminate methods that are less likely to work well in practice. Then we use the \emph{sepsis treatment simulator} in a series of experiments to test and adjust the methods towards real-life application. The simulator is constructed from real-world data in accordance with the theoretical assumptions of COIRL. Throughout the experiments we strip the assumptions from the simulator and show that the methods perform well in an offline setting. Furthermore, we show that a DNN estimator achieves high performance when the mapping from the context to the reward is non-linear. \\

\noindent
Finally, we test the methods in \emph{sepsis treatment} -- without the simulator. Here, we use real clinicians' trajectories for training and testing. For COIRL, we estimate a CMDP\textbackslash M model from the train data (states and dynamics) which is used for training purposes. We then show that COIRL achieves high action matching on unseen clinicians trajectories.


\subsection{Grid world}\label{subsec:exp_irl}

\begin{wrapfigure}{r}{0.35\textwidth}
    \vspace{-1cm}
  \begin{center}
    \includegraphics[width=1.0\linewidth]{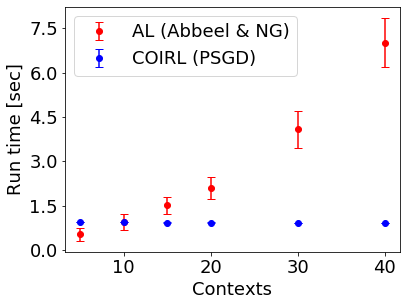}
  \end{center}
  \caption{\textbf{Run-time comparison between COIRL and AL}. AL run-time grows as the number of contexts grows while COIRL run-time stays fixed.}
  \label{fig: AL_runtime}
\end{wrapfigure}

The grid world domain is an $n$ by $m$ grid which makes $|\mathcal{S}|=n\cdot m$ states. The actions are $\mathcal{A}=\{left,up,right,down\}$ and the dynamics are deterministic for each action, i.e., if the action taken is up, the next state will be the state above the current state in the grid (with cyclic transitions on the borders, i.e., taking the action $right$ at state $(n-1, y)$ will transition to $(0, y)$). The features are one-hot vectors ($\phi(s_i)=e_i \in\mathbb{R}^{n\cdot m}$). The contexts correspond to "preferences"
of certain states on the grid. The contexts are sampled from a uniform distribution over the $n\cdot m$ dimensional simplex.
\\

\noindent
This domain is used to evaluate the application of IRL to COIRL problems. We compare the performance of PSGD (COIRL) and the projection algorithm (AL) of \cite{abbeel2004apprenticeship} as a function of the context space size. This framework is applied on a grid with dimensions of $3\cdot4$, overall 12 states. The PSGD method trains on a CMDP model and the projection algorithm trains on a large MDP model, with a state space that includes the contexts, as noted in \cref{subsec:largemdp}. The new states are $s' = (s,c)$, and the new features are $\phi(s')=c\odot \phi(s)$. We measure the run-time of every iteration. The most time consuming part of both methods is the optimal policy computation time for a given reward. Both methods use the same implementation of value iteration in order to enable a comparison of the run-time. 
\\

\noindent
The results shown in \cref{fig: AL_runtime} show that the projection algorithm in the large MDP requires significantly more time to run as the number of contexts grows, while the run-time of PSGD is not affected by the number of contexts. 
\\


\noindent
\textbf{Conclusion:} applying IRL methods in a large MDP environment limits the number of contexts that can be used, and as seen in the results, its run time grows when the number of contexts increases. We conclude that applying IRL to COIRL problems is inefficient and exclude this method from the following experiments.

\subsection{Autonomous driving simulation}
\label{subsec:car}

\begin{wrapfigure}{r}{0.3\textwidth}
  \vspace{-1.5cm}
  \begin{center}
    \includegraphics[width=1.0\linewidth]{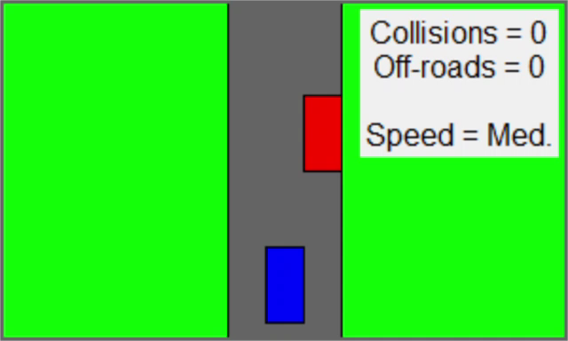}
  \end{center}
  \vspace{-0.5cm}
  \caption{Driving simulator.}\label{fig:car}
  \vspace{-0.7cm}
\end{wrapfigure}

While the grid world focused on comparing COIRL with the standard IRL method, in this section we compare the various methods for performing COIRL in an autonomous driving simulator.
This domain involves a three-lane highway with two visible cars, cars A and B. The agent, controlling car A, can drive both on the highway and off-road. Car B drives in a fixed lane, at a slower speed than car A. Upon leaving the frame, car B is replaced by a new car, appearing in a random lane at the top of the screen. The features denote the speed of car A, whether or not it is on the road and whether it has collided with car B. The context implies different priorities for the agent; should it prefer speed or safety? Is going off-road a valid option? For example, an ambulance will prioritize speed and may drive off-road, as long as it goes fast and avoids collisions, while a bus will prioritize avoiding both collisions and off-road driving as safety is its primary concern. The mapping from the contexts to the true reward is constructed in a way that induces different behaviors for different contexts, making generalization a challenging task.

\subsubsection{Ellipsoid setting}
\label{subsec:ellipsoid_setting}
The ellipsoid method requires its own framework. Here, the agent's policy is evaluated by an expert for every new context revealed. Only if its value is not $\epsilon$-close to the optimal policy value, an expert demonstration will be provided (feature expectations of an expert for the revealed context). While the ellipsoid method can only perform a single update for each demonstration, the descent methods can utilize all of the previously revealed demonstrations and perform update steps until convergence. We measure the accumulated amount of expert demonstrations given at each time-step and the value of the agent on a holdout test set, for each new given demonstration.
\\

\noindent
The amount of given demonstrations is important in the ellipsoid framework, as it is equal to the number of times that the agent is not $\epsilon$-close to the optimal policy value. In addition, it is a way to measure how much intervention is required by an external expert. We would expect a `good' method to be $\epsilon$-optimal for most revealed contexts and therefore it should observe a small amount of demonstrations.
\\

\noindent
The results, presented in \cref{fig: ellipsoid results}, show that all methods eventually reach the expert's value; however, the descent methods are more sample efficient than the ellipsoid method and require fewer expert demonstrations. While according to the theoretical guarantees (\cref{table:discussion}, feature expectations setting) the ellipsoid method should have better sample complexity, in practice it is surpassed by the results of the descent methods. Note that in this experiment each demonstration may be used more than once by the descent methods, hence the theoretical results are not valid for them.

\begin{figure*}[h]
    \centering
    \subfigure[\# demonstrations]{\label{fig:results demonstrations}\includegraphics[width=0.45\linewidth]{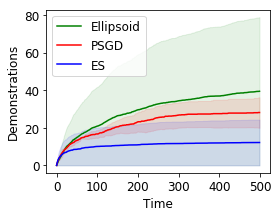}}
    \subfigure[Value]{\label{fig:results generalization}\includegraphics[width=0.474\linewidth]{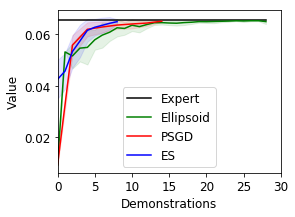}}
    \vspace{-0.3cm}
    \caption{\textbf{Comparison of the ellipsoid method with the ES and PSGD methods in the autonomous driving simulation}. The graph on the left compares the number of demonstrations required by each method, while the graph on the right compares the performance at each time-step. We observe that while, as theoretically shown, all methods eventually find an $\epsilon$-optimal solution, the descent methods attain better sample efficiency (converge faster and require less expert interaction).} 
    \label{fig: ellipsoid results}
\end{figure*}

\subsubsection{Online setting}
\label{subsub:car_online}

Here, we compare the descent methods presented in \cref{sec:methods} in an online setting. Each descent step is performed on a context-$\mu$ pair, where the context is sampled 
uniformly from the simplex
and $\mu$ is the feature expectations of a policy that is optimal for this context. For each method, we measure the normalized value of the proposed policies with respect to the real reward, the loss (\cref{eq:loss}), and the accuracy, which represents how often the expert and agent policies match. These criteria are evaluated on a holdout set of contexts, unseen during training. The x-axis corresponds to the number of contexts seen during training, i.e., the number of subgradient steps taken.
\\

\noindent
In this setting we use two setups, which differ by the observed feature expectations. First, in the \textbf{feature expectations} setup, we assume that the whole optimal policy can be observed, therefore, for training we use the feature expectations of the expert's policy. The results are shown in \cref{fig: online_car_results_feat_exp}. They show a strong correlation between `loss minimization' and `value maximization'. EW converges faster than PSGD and the ES method consistently lies between EW and PSGD, displaying comparable sample complexity. These results match the theoretical guarantees (\cref{table:discussion}, feature expectations) as EW has tighter bounds when it comes to scalability compared to PSGD and ES.
\\

\begin{figure*}[h]

    \centering
    \subfigure[Loss]{\label{fig:car_online_los_feat_exp}
    \includegraphics[width=0.3\linewidth]{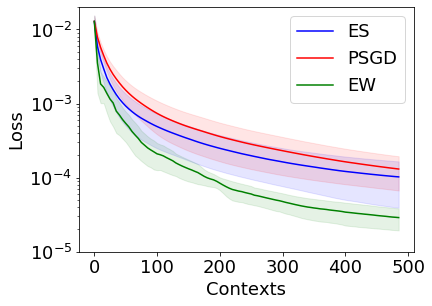}}
    \subfigure[Value]{\label{fig:car_online_val_feat_exp}
    \includegraphics[width=0.29\linewidth]{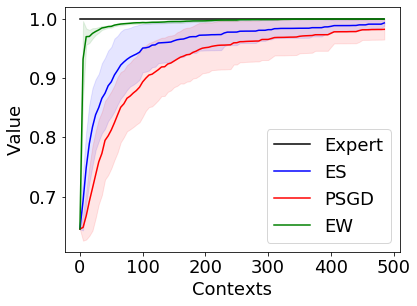}}
    \subfigure[Accuracy]{\label{fig:car_online_acu_feat_exp}
    \includegraphics[width=0.28\linewidth]{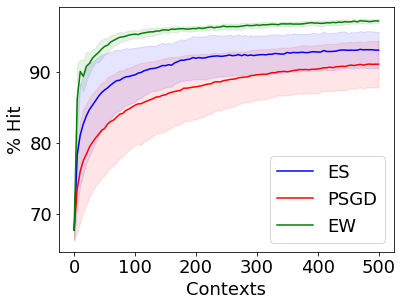}}
    \caption{\textbf{Online learning curve in the autonomous driving simulation -- learning from feature expectations}. The expert demonstrations are provided in the form of the feature expectations of the expert's policy. We compare the loss, value and accuracy, where the value and accuracy are relative to the expert's behavior. As can be seen, all descent methods minimize the loss and achieve high value. Additionally, we observe that while they do attain relatively high accuracy, they find policies which are optimal yet differ from the expert in the actions taken.}
    \label{fig: online_car_results_feat_exp}
\end{figure*}

\noindent
The second setup we use is the \textbf{trajectories} setup. Here we construct the feature expectations using a finite number of samples taken from the expert's policy, each context correspond to a finite rollout of an expert (motivated by real life limitations). The results in \cref{fig: online_car_results} show that all three descent methods attain high value and accuracy in this setup. As in the feature expectations setting, the results validate the theoretical sample complexity, with the exception that ES performs slightly better than PSGD. Comparing the results of the different setups we observe similar performance for training with the whole expert's policy or a sample of it, as expected (\cref{subsec:mda}, practical MDA). Training with trajectories is closer to the available data in real-life applications, since 
only samples of policies are provided.
\\

\begin{figure*}[h]
    \centering
    \subfigure[Loss]{\label{fig:car_online_los}
    \includegraphics[width=0.3\linewidth]{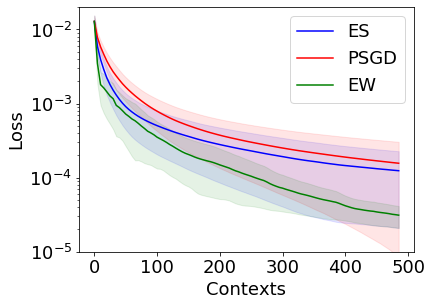}}
    \subfigure[Value]{\label{fig:car_online_val}
    \includegraphics[width=0.29\linewidth]{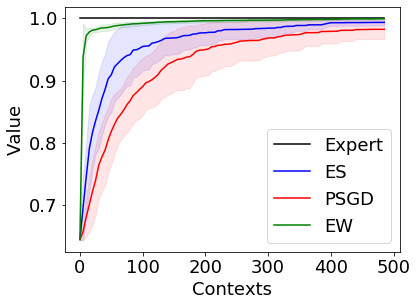}}
    \subfigure[Accuracy]{\label{fig:car_online_acu}
    \includegraphics[width=0.28\linewidth]{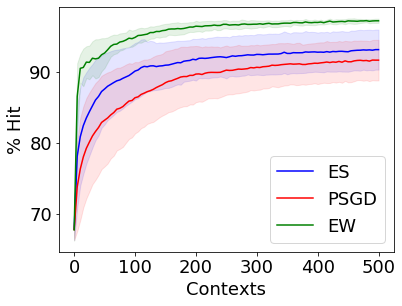}}
    \caption{\textbf{Online learning curve in the autonomous driving simulation -- learning from trajectories}. While in \cref{fig: online_car_results_feat_exp} the demonstrations were in the form of feature expectations, here we provide trajectories, a less informative approach. Although less informative, we observe that, similarly to \cref{fig: online_car_results_feat_exp}, all methods perform well, attaining similar performance as when given the full information.}
    \label{fig: online_car_results}
\end{figure*}

\noindent
\textbf{Conclusion:} The ellipsoid method is not as sample efficient as the descent methods. Furthermore, it demands constant expert monitoring, which in real-world problems might be unavailable. In many real-world tasks, such as the sepsis treatment domain, there is an abundance of offline data, yet evaluation in real-life may not be available. Thus, we do not include experiments of the ellipsoid method in the sepsis treatment domain.\\

\noindent
The ES and EW methods also have their drawbacks: ES requires computation of the loss function at a considerably large number of points for every descent step. This requirement makes the ES method computationally expensive and prevents it from scaling to larger environments. The EW method assumes that the model parameters lay within the simplex, an assumption that limits the policy space in the linear case, and may not hold in the non-linear case, where the mapping between the context and the reward is modeled by a neural network. As such, we do not include these methods in the sepsis treatment domain.

\subsection{Sepsis treatment simulator}
\label{subsec:med_rew}

This domain simulates a decision-making process for treating sepsis. Sepsis is a severe, life-threatening infection, where the treatment applied to a patient is crucial for saving its life. To create a sepsis treating simulator, we leverage the MIMIC-III data set \citep{Johnson_2016}. This data set includes data from hospital electronic databases, social security, and archives from critical care information systems, that had been acquired during routine hospital care. We follow the data processing steps that were taken in \cite{1902.03271} to extract the relevant data in a form of normalized measurements of sepsis patients during their hospital admission and the treatments that were given to each patient. The measurements include dynamic measures, e.g., heart rate, blood pressure, weight, body temperature, blood analysis standard measures (glucose, albumin, platelets count, minerals, etc.), as well as static measures such as age, gender, re-admission (of the patient), and more.
\\

\noindent
From the processed data we construct a dynamic treatment regime, modeled as a CMDP, in which a clinician acts to improve a sick patient's medical condition. The context represents patient features that are constant during treatment, such as age and height. The state summarizes dynamic measurements of the patient, e.g., blood pressure and EEG readouts. The actions represent different combinations of fluids and vasopressors, drugs commonly provided to restore and maintain blood pressure in sepsis patients. The mapping from the context to the true reward is constructed from the data.  Dynamic treatment regimes are particularly useful for managing chronic disorders and fit well into the broader paradigm of personalized medicine \citep{komorowski2018artificial,prasad2017reinforcement}. Furthermore, dynamic treatment regimes have contextual properties; what is defined as healthy blood pressure for a patient differs based on age and weight \citep{wesselink2018intraoperative}. In our setting, $W^*$ captures this information -- mapping from contextual (e.g., age) and dynamic information (e.g., blood pressure) to reward.
\\

\noindent
As noted in previous sections, we move toward real-life application and eliminate the inefficient methods. In this section we evaluate the PSGD and compare it with GPI (\cref{sec: transfer}) and contextual BC (\cref{subsec:contextual_poli}).

\subsubsection{Online setting}
\label{subsub:med_online}

In this setting we evaluate only the PSGD method. Similarly to the autonomous driving simulation we use two setups: (1) we train the methods with the expert's feature expectations for each context, and (2) instead of using the expert's feature expectations for each given context, we use an estimation, calculated from a given expert trajectory (\cref{subsec:mda}, practical MDA). We present the results of both setups in the same figure, so a comparison between the setups can be done.
\\

\noindent
We observe in \cref{fig: results_med} that PSGD performs well in both setups, with slightly better performance with feature expectations, as expected. 
 This supports the theory, as using samples should not affect the convergence results and truncation after 40 steps should incur only a small penalty.
An important observation is that high accuracy is not necessary for high value, as our agents achieve near-perfect value with relatively low accuracy. This reinforces the use of IRL for imitation learning tasks, as it supports the claim that the reward function, not the policy, is the most concise and accurate representation of the task \citep{abbeel2004apprenticeship}. 


\begin{figure*}[h]

    \centering
    \subfigure[Value]{\label{fig: medical value} \includegraphics[width=0.47\linewidth]{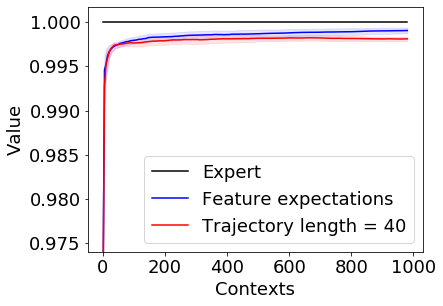}}
    \subfigure[Accuracy \%]{\label{fig: medical hit} \includegraphics[width=0.43\linewidth]{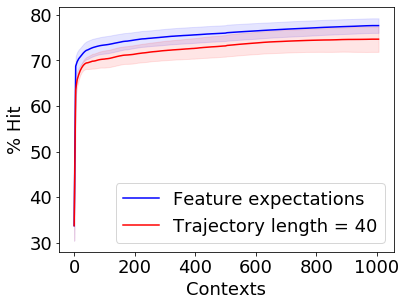}}
    \caption{\textbf{Online setting in sepsis treatment}. We compare the relative value and accuracy when the agent is provided the feature expectations or finite length trajectories. We observe that while as the feature expectations are more informative, the performance is slightly better. However, notice that the difference is negligible and amounts to less than 0.5\% difference in the relative value.}
    \label{fig: results_med}
\end{figure*}

\subsubsection{Offline setting}\label{subsubsec:exp_bc}
Here, we evaluate the COIRL, GPI and contextual BC methods. We test the ability of these methods to generalize with a limited amount of data. The motivation for this experiment comes from real-life applications, where the data available is often limited in size. The data, similarly to the online setting, is constructed from context-trajectory pairs. In this setting we minimize the loss function (\cref{eq:loss}) by taking descent steps on mini-batches sampled from the data set, with repetition, which invalidates the theoretical results. We conduct two experiments that evaluate the performance as a function of the train-set size (the amount of context-trajectory pairs used for training). We consider two mappings from the context to the reward; a \textbf{linear} mapping, and a \textbf{non-linear} mapping. For the \textbf{non-linear} mapping we use a DNN estimator which constitutes another step towards real-world applicability.

\begin{remark}
\label{rmk:bc}
Contextual BC is a method to learn a contextual policy, instead of a contextual reward. In its implementation we use a DNN that, given a \emph{context} and \emph{state-features}, computes a probability vector, $\hat{\pi}_c(s)$, representing the agent's policy -- i.e., the probability to take action $a\in \mathcal{A}$ is the $a$'th element of the DNN output $\hat{\pi}_c(s, a)$. The \emph{state-features} that are given as an input greatly affect BC performance, especially when we compare it to COIRL, which uses the real dynamics as well as features that represent each state. BC can make good use of the dynamics, as states with similar dynamics should be mapped to similar actions. To improve the performance for BC, we use the same \emph{state-features} that COIRL uses (HR, blood pressure, etc...), in addition to a feature-vector that represents the dynamics. For each state, $s\in \mathcal{S}$, the dynamics can be represented as a concatenation of the probability vectors, $\big\{P(s,a)\big\}_{a\in \mathcal{A}}$, where $P(s,a)[i]=P(s,a,s_i)$. The dimension of the dynamics for each state is $|\mathcal{S}|\cdot|\mathcal{A}|$ which is relatively large in the sepsis treatment simulator, hence we reduce its dimensionality with PCA.
\end{remark}

\begin{figure*}[h]
    \centering
    \subfigure[Value]{\label{fig:offline_med_value}\includegraphics[width=0.45\linewidth]{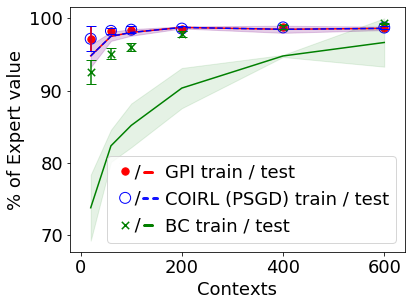}}
    \subfigure[Accuracy]{\label{fig:offline_med_accuracy}\includegraphics[width=0.45\linewidth]{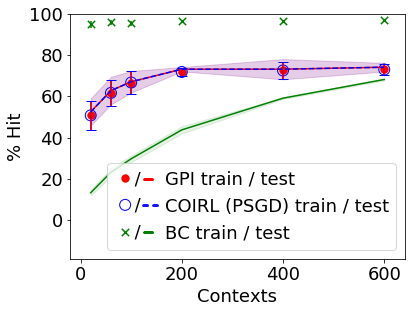}}
    \vspace{-0.3cm}
    \caption{\textbf{Offline setting in sepsis treatment}. The x-axis denotes the number of contexts in the training set. Results on the train data are represented using circles and x's, the results on a holdout test data-set represented as lines. Given a sufficient amount of contexts seen, GPI is comparable to re-solving the domain, hence there is a large overlap between the results of GPI and COIRL. Contextual BC requires much more data to generalize well.}
    \label{fig: offline results}
\end{figure*}

\noindent
In \cref{fig: offline results} we compare the performance of COIRL, GPI and contextual BC in the \textbf{linear} setting, when provided with a fixed amount of data. The results show that in the sepsis treatment domain, the COIRL and GPI methods perform similarly and able to generalize well for a small amount of train data compared to contextual BC. As expected, in \cref{fig:offline_med_accuracy} BC attains better accuracy on the train data while in \cref{fig:offline_med_value} COIRL and GPI methods attain better value on the train data. Another observation is that COIRL achieves similar performance on the training data and on the test data; it is able to generalize to unseen contexts, even when the amount of training data is small. On the other hand, BC achieves almost perfect accuracy and high value on the train data but performs poorly on the test data. This generalization gap goes away only when a large amount of data is available for training.
\\

\begin{figure*}[h]
    \centering
    \subfigure[Value]{\label{fig:offline_med_value_NL}\includegraphics[width=0.45\linewidth]{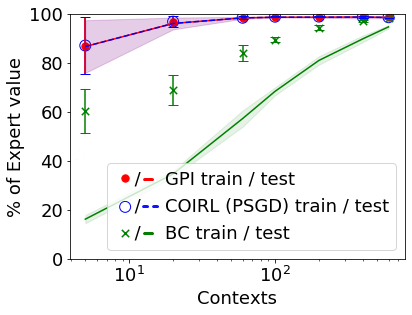}}
    \subfigure[Accuracy]{\label{fig:offline_med_accuracy_NL}\includegraphics[width=0.45\linewidth]{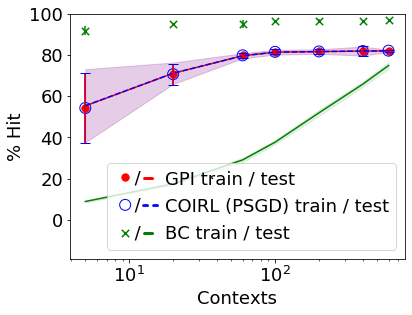}}
    \vspace{-0.3cm}
    \caption{\textbf{Offline setting in sepsis treatment: non-linear mapping}. The x-axis denotes the number of contexts in the training set (logarithmic scale). Results on the train data are represented using circles and x's, the results on a holdout test data-set represented as lines. Similar to the linear setup, GPI and COIRL generalize well for a small amount of train data where the performance on the train data and on the test data is similar. Contextual BC performance on the train set is almost perfect, where its performance on the test data requires a large amount of expert demonstrations.}
    \label{fig: offline results_NL}
\end{figure*}

\noindent
The \textbf{non-linear} setup results presented in \cref{fig: offline results_NL}. Here, the x-axis is in logarithmic scale. The performance of all methods is similar to the \textbf{linear} setup; COIRL and GPI methods perform similarly and generalize to unseen contexts even when given a small amount of train data. Contextual BC generalizes to unseen contexts only for a large amount of train data. As in the \textbf{linear} setup, the BC method attains better accuracy while the COIRL and GPI methods attain better value.

\subsection{Sepsis treatment in real-life}
\label{subsec:med_real}
In the previous subsections we focused on analyzing COIRL in simulated environments. We have taken a sequence of steps with the aim of making the simulations more and more realistic. In all of these simulations, the expert trajectories were always generated from the optimal policy (for a given context) w.r.t to the true context-reward mapping. Our results suggest that the reward estimated by COIRL induces a policy that attains a close-to-expert value in both linear and non-linear settings. Now we turn to examine our algorithms in a real world data set. Since the true mapping is no longer known, we can only measure the accuracy of our resulted policies. In previous sections we observed that while accuracy does not necessarily imply value (i.e., a policy can have optimal value but not be $100\%$ accurate), these measures are often correlated. In addition, since the true dynamics of the MDP is now unknown, we estimate it from the data itself. \\

\noindent
\textbf{Data processing}:
We follow the steps done in \cite{komorowski2018artificial} to construct a time-series data of static and dynamic measurements. The data is divided to trajectories, where each trajectory represent a different patient. We consider only trajectories of length greater than $10$ that represent $40$ hours. The processed data is consisted of $14,591$ trajectories, divided to a 60-20-20 train-validation-test partition. Each trajectory corresponds to a static measurements vector and a time series of dynamic measurements vectors, with time steps of 4 hours. In the following experiments each model is trained on the training set, until an early stopping criteria is met on the validation set. We then report the accuracy (action matching with the clinicians actions) on the holdout test set. 
\\

\noindent
\textbf{Model fitting}:
As in \cref{subsec:med_rew}, the contexts and the states constructed from static and dynamic measurements respectively. In our model, the contexts are in $\mathbb{R}^7$ and include the gender, age, weight, GCS, elixhauser co-morbidity score, whether the patient was mechanically ventilated at $s_0$ and whether the patient hs been re-admitted to the hospital. The actions are defined to be the amount of vasopressors given to a patient at each time slot, and five discrete actions are constructed by dividing the possible values into five bins. The state space is constructed by clustering the observed patient dynamic measurements from the data with K-means \citep{macqueen1967some}. The clustering process is repeated for different numbers of states and different weights for each measurement (to control the importance of each measurement for the state space). Each model is evaluated by two terms: (1) number of different actions taken on the same state for the same patient: $\mathbb{E}_\tau\big[\mathbb{E}_{s\in\tau}[ |\hat{\mathcal{A}}_s^{\tau}|]\big]$, where $\hat{\mathcal{A}}_s^{\tau}=\{a\in\mathcal{A} : (s,a) \in \tau\}$. (2) number of different states in each trajectory: $\mathbb{E}_\tau\big[ |\hat{\mathcal{S}}^\tau| \big]$, where $\hat{\mathcal{S}}^\tau = \{s\in\mathcal{S} : s \in \tau\}$. In both terms, $\tau$ is a trajectory drawn from the data. We require the first term to be as small as possible, to achieve a consistent experts policies in the CMDP model, the second term required to be large, to force the resulted model to distinguish between different states in the same patient's trajectory. Obviously, the model has to be as small as possible, to enable generalization. The chosen model consists $5000$ states. 
\\
 
\noindent
While processing the data, we noticed that clinicians behavior with respect to some measurements is random. To address this matter we consulted with clinicians and defined a set of important dynamic measurements, among them we use the clustering process to choose the patients relevant dynamic measurements for the states; states were clustered for any possible single measurement and the five best dynamic measurements were chosen: mean blood pressure, diastolic blood pressure, shock index, cumulative balance of fluids and the fluids given to a patient. The features in this CMDP are action-dependent and set to be a concatenation of $e_i \in \mathbb{R}^{|\mathcal{S}|}$ and $e_j \in \mathbb{R}^{|\mathcal{A}|}$ where $e_i$ is a vector of all zeros and a single 1 that represents each state and $e_j$ represents the action, overall the there are $5,005$ features for each state-action pair.
\\

\noindent
As described in \cref{sec:prelim}, learning the transition kernel is an orthogonal problem to the COIRL problem, and can be viewed as a part of the model fitting process. Our dynamics model is context-dependent; the contexts (patients) clustered into five clusters and the dynamics of each cluster are then estimated using the training data. \\
\noindent
\textbf{Methods}:
For COIRL we report results for the \emph{linear} and the \emph{non-linear} mappings. In both setups, we use a discount factor $\gamma=0.7$ and a mini-batch of size $32$. The stopping criteria is set to stop when five consecutive steps do not increase the validation accuracy. To speed-up the validation process we sample a subset of $300$ patients from the validation partition at the beginning of each seed and use them to validate the model. In the \emph{linear} setup the step size is $\alpha_t = 0.25\cdot 0.95^t$. The \emph{non-linear} setup use a DNN to learn the mapping $f_W:\mathcal{C}\longrightarrow \mathbb{R}^{|S|+|A|}=\mathbb{R}^{5,005} \approx \mathbb{R}^{5K}$, it has four layers with a Leaky ReLU activation and batch-normalization between the first and second layers, and Leaky ReLu activation between the second and third layers, their sizes are $20K$,$10K$,$10K$,$5K$ respectively. Here, the step size is $\alpha_t = 0.2\cdot 0.95^t$
\\

\begin{wraptable}{r}{0.4\textwidth}
\vspace{-0.5cm}
    \setlength\extrarowheight{8pt}
\resizebox{0.4\textwidth}{!}{\begin{tabular}{|c|c|c|}
    \hline
    \multicolumn{2}{|c|}{\textbf{Method}} & \textbf{Accuracy \%}  \\ \hline
    \multirow{2}{*}{COIRL} & Non-linear & $\mathbf{83.74 \pm 1.0}2$ \\ \cline{2-3}
    & Linear & $45.17 \pm 7.14$ \\ \hline
     \multicolumn{2}{|c|}{BC} & $73.12 \pm 0.82$ \\ \hline
\end{tabular}}
\caption{Results on real world data. We measure the accuracy of each method over a holdout test set. In the non-linear setting, COIRL achieves the best accuracy and outperforms BC.}
\label{table:real_world_res}
\vspace{-0.5cm}
\end{wraptable}

\noindent
For BC we also use a DNN for function approximation, as we found it to work much better than a linear model. We also experimented with different sets of features as inputs. The features that we found to give the best performance were computed in a similar manner to the features that we used for BC in \cref{rmk:bc}, using the dynamics of the estimated CMDP, resulted with $5K$ features that represent each state. Concretely, the DNN received a concatenation of the context and the features that represent the current state (size of $5,007$) and outputs a stochastic policy (softmax over the outputs of the last layer). The network architecture is composed of three linear layers of sizes $625$,$125$,$5$. Each layer is followed by a Leaky ReLU activation, and a Softmax activation is used on the output. Similar to COIRL, the model is trained over the training set partition and the stopping criteria is set to stop after 5 epochs of non-increasing validation accuracy. The loss of the DNN is the binary cross-entropy loss between the DNN output and the observed action, $e_i \in \mathbb{R}^{|\mathcal{A}|}$. The mini-batch size is $32$ and the optimizer is SGD with step size $\alpha_t = 0.1\cdot \frac{1}{1+10^{-7}t}$.
\\
\noindent
Each method trained and evaluated over five seeds, the results are presented in \cref{table:real_world_res}. We can see that COIRL with a non-linear mapping attains the best performance, while the linear mapping achieves poor accuracy. BC performs well overall, but not as good as COIRL. In \cite{lee2019truly} the authors use similar data set and action space. Their methods, TRIL and DSFN, achieve $80 \pm 2\%$ and $79 \pm 5\%$ respectively, which is lower then COIRL and with higher variance. These results suggest that the contextual hypothesis better represents the real world, i.e., that physicians indeed use personalized treatments based on context. 

\vspace{-0.2cm}
\section{Discussion}\label{sec:dis}

\noindent
Motivated by current trends in personalized medicine \citep{mit_technology_review_2020}, we proposed the Contextual Inverse Reinforcement Learning framework. While most works in RL assume the agent is provided with a reward signal, we focused on a more realistic setting, where the reward is unknown to the agent and, instead, it observes and receives feedback from an expert. As opposed to the standard IRL setting, in the contextual case, each context defines a new MDP. This leads to a new form of generalization in RL, where the agent is trained and learns how to act optimally on a set of contexts, followed by an evaluation procedure on a set to which the agent was not exposed during training.\\

\noindent
We show that solving the COIRL objective can be performed by minimizing a convex optimization task. As this objective is not differentiable, we proposed two schemes based on subgradient descent (MDA and ES) and an adaptation of cutting plane methods (ellipsoid). We analyzed the convergence properties of each algorithm and summarized the results in \cref{table:discussion}.\\

\noindent
All of the proposed methods assume that the dynamics are known, but in many applications the dynamics and even the state space are unknown. Following the description in \cref{sec:prelim}, any method that learns the dynamics efficiently can be used prior to COIRL. For example, in online frameworks, where the expert provides demonstrations in an online manner, the dynamics can be learned as proposed in \citet{10.1145/1102351.1102352}. In this case, the dynamics estimation and COIRL should run iteratively, such that every change in the estimation of the dynamics introduces a new COIRL problem that should be solved. In offline frameworks the dynamics can be estimated prior to COIRL, similarly to \cref{subsec:med_real}. \\

\noindent
In addition to the theoretical analysis, we performed extensive empirical evaluation between all proposed algorithms, including baseline approaches. Here, we see a mixed correlation between theoretical and practical results. Regarding the ellipsoid schemes, we observe that indeed as shown theoretically, they are sub-optimal compared to the other methods. However, comparing MDA to ES, we see that ES matches and sometimes outperforms MDA even though the theoretical upper-bounds are tighter for MDA. These results correlate with observations seen by \cite{nemirovsky1983problem}, where ES often provides better empirical results.\\

\noindent
Aside from comparing between our proposed methods, we also compared to a common learning scheme -- behavioral cloning. While IRL aims to find a reward function which explains the experts behavior, behavioral cloning (log-likelihood action matching) simply converts the RL task into a supervised learning problem. Previous works \citep{abbeel2004apprenticeship} talk about the importance of IRL, compared to BC. In our experiments we see this clearly. While the reward/value is smooth (Lipschitz) w.r.t. the context, the policy is not. As a small change in the context may lead to a large switch in the policy (the optimal actions change in certain states), we observe that BC struggles. This can also be seen in the fact that COIRL often reaches imperfect action-matching (accuracy) yet attains the optimal return.\\

\noindent
We demonstrated how existing policies can be transferred to new contexts, avoiding planning in test-time. This is important, as planning complexity is a function of the size of the MDP, thus this form of transfer may be crucial for real-world scenarios. Our experiments illustrate how combining offline COIRL with GPI eases the computational load on the agent while maintaining strong performance with few training examples.\\

\noindent
Finally, COIRL achieved the highest accuracy in the challenging task of predicting the clinicians treatment in the real world sepsis treatment data set. This suggests that sepsis treatment can be modeled as a contextual MDP; we hope that these findings will motivate future work in using contextual MDPs to model real-world decision making.\\

\noindent
To conclude, we proposed the COIRL framework and analyzed it under a linear mapping assumption. In real-world cases, where the linear assumption holds, COIRL can be used effectively. Future work may combine COIRL with schemes such as meta-learning \citep{finn2017model} in order to cope with infinitely large MDPs and non-linear mappings.

\bibliography{bibliography.bib}
\bibliographystyle{icml2020}      

\appendix
\appendixpage

\section{Proofs for Section 3}
\label{app:4}
\begin{definition}[Bregman distance]
\label{def:Bregman}
    Let $\psi: \mathcal{W} \rightarrow R$ be strongly convex and continuously differential in the interior of $\mathcal{W}$. The Bregman distance is defined by 
     $D_\psi(x,y)  = \psi(x) - \psi(y) - ( x -y)\cdot \nabla \psi(y),
    $
    where $\psi$ is strongly convex with parameter $\sigma$.
\end{definition}

\begin{definition}[Conjugate function]
\label{def:conj}
    The conjugate of a function $\psi(y)$, denoted by  $\psi^*(y)$ is $$\max_{x\in\mathcal{W}}\left\{ x \cdot y - \psi(x)\right\}.$$
\end{definition}

\noindent
\textbf{Example}: let $\|\cdot \|$ be a norm on $\mathbb {R} ^{n}.$ The associated dual norm, denoted $\|\cdot \|_{*},$ is defined as $\|z\|_{*}=\sup\{z^{\intercal }x\;|\;\|x\|\leq 1\}.$ The dual norm of $\|\cdot \|_2$ is $\|\cdot \|_2$, and the dual norm of $\|\cdot \|_1$ is $\|\cdot \|_\infty$.\\
\noindent
Before we begin with the proof of \cref{thm:covexloss}, we make the following observation. By definition, $\hat \pi_c (W)$ is the optimal policy w.r.t.~$c^TW;$ i.e., for any policy $\pi$ we have that
\begin{equation}
    \label{eq:optimallity}
    c^T W \cdot \mu^{\hat \pi_c (W)}_c \ge c^TW \cdot \mu^\pi_c.
\end{equation}



\begin{proof}[Proof of \cref{thm:covexloss}]~\\
\noindent
\textbf{1.} We need to show that $\forall W_1,W_2 \in \mathcal{W},\forall \lambda\in [0,1], $ we have that $$L_{\text{lin}}(\lambda W_1 + (1-\lambda)W_2) \le  \lambda L_{\text{lin}}(W_1)+(1-\lambda)L_{\text{lin}}(W_2)$$ 
\begin{align*}
  &L_{\text{lin}}(\lambda W_1 + (1-\lambda)W_2) \\ &= \mathbb{E}_c \left[ c^T(\lambda W_1 + (1-\lambda)W_2) \cdot  \Big(\mu^{\hat{\pi}_{c}\big(
     \lambda W_1 + (1-\lambda)W_2\big)}_c - \mu^*_{c}\Big)\right]\\
  &= \lambda \mathbb{E}_c \left[ c^TW_1 \cdot \Big(\mu^{\hat{\pi}_{c}\big(
     \lambda W_1 + (1-\lambda)W_2\big)}_c - \mu^*_{c}\Big) \right]\\
  & \;\;\;\; + (1-\lambda) \mathbb{E}_c \left[ c^TW_2 \cdot \Big(\mu^{\hat{\pi}_{c}\big(
     \lambda W_1 + (1-\lambda)W_2\big)}_c - \mu^*_{c}\Big)\right] \\
  &\le \lambda \mathbb{E}_c \left[ c^TW_1 \cdot \Big(\mu^{\hat{\pi}_{c}\big( W_1 \big)}_c - \mu^*_{c}\Big) \right]
  + (1-\lambda) \mathbb{E}_c \left[ c^TW_2 \cdot \Big(\mu^{\hat{\pi}_{c}\big(W_2\big)}_c - \mu^*_{c}\Big)\right] \\
  &= \lambda L_{\text{lin}}(W_1)+(1-\lambda)L_{\text{lin}}(W_2),  
\end{align*}
where the inequality follows from \cref{eq:optimallity}.

\noindent
\textbf{2.} 
Fix $z \in \mathcal{W}.$ We have that  
\begin{align*}
L_{\text{lin}}(z) &= \mathbb{E}_c \left[ c^Tz \cdot  \Big(\mu^{\hat{\pi}_{c}(
     z)}_c - \mu^*_{c}\Big)\right] \\
& \ge \mathbb{E}_c \left[ c^Tz \cdot  \Big(\mu^{\hat{\pi}_{c}(
     W)}_c - \mu^*_{c}\Big)\right]\\ 
&= L_{\text{lin}}(W) + (z - W) \cdot \mathbb{E}_c \left[c\odot \Big(\mu^{\hat{\pi}_{c}(
     W)}_c - \mu^*_{c}\Big) \right],
\end{align*}
where the inequality follows from \cref{eq:optimallity}.

\noindent
\textbf{3.}
Recall that a bound on the dual norm of the subgradients implies Lipschitz continuity for convex functions. Thus it is enough to show that $\forall W \in \mathcal{W}, \| g(W)\|_p =  \| \mathbb{E}_c \left[c\odot \Big(\mu^{\hat{\pi}_{c}(W)}_c - \mu^*_{c}\Big) \right] \|_p \le L.$ Let $p=\infty$, we have that
\begin{align}
\| g(W)\|_\infty & = \left\| \mathbb{E}_c c\odot \Big(\mu^{\hat{\pi}_{c}(W)}_c - \mu^*_{c}\Big)  \right\|_\infty  \nonumber \\
& \le \mathbb{E}_c \| c\odot \Big(\mu^{\hat{\pi}_{c}(W)}_c - \mu^*_{c}\Big)  \|_\infty \tag{Jensen inequality} \\
& \le \mathbb{E}_c \| c \|_\infty \norm{\mu^{\pi_i}_c-\mu^{\pi_j}_c}_\infty \le \frac{2}{1-\gamma} \label{eq:L}.
\end{align}
where in \cref{eq:L} we used the fact that $\forall \pi$ we have that $\norm{\mu^\pi_c}_\infty \le \frac{1}{1-\gamma}, $ 
thus, for any $\pi_i,\pi_j,$ $$\norm{\mu^{\pi_i}_c-\mu^{\pi_j}_c}_\infty \le \frac{2}{1-\gamma}.$$ 
\noindent
Therefore, $L = \frac{2}{1-\gamma}$ w.r.t.~$\norm{\cdot}_\infty$. Since $\norm{\cdot}_2 \le \sqrt{dk}\norm{\cdot}_\infty $  we get that $L = \frac{2\sqrt{dk}}{1-\gamma}$ w.r.t.~$\norm{\cdot}_2$ . 
\end{proof}

\section{Proofs for Section 3.2}\label{sec:ellipsoid_proofs}

\subsection{Proof of Theorem~\ref{thm: upper bound}}
\begin{proof}[Proof of Theorem ~\ref{thm: upper bound}]
We prove the theorem by showing that the volume of the ellipsoids $\Theta_t$ for $t=1,2,...$ is bounded from below. 
In conjunction with Lemma~\ref{lemma: ellipsoid}, which claims there is a minimal rate of decay in the ellipsoid volume, this shows that the number of times the ellipsoid is updated is polynomially bounded.\\

\noindent
We begin by showing that $\underline{W}^*$ always remains in the ellipsoid. We note that in rounds where $V^*_{c_t} - V^{\hat\pi_t}_{c_t} > \epsilon$, we have 
$\underbar{$W^*$}^T \Big(c_t\odot\Big(\mu^*_{c_t}-\mu^{\hat\pi_t}_{c_t}\Big)\Big) > \epsilon$. In addition, as the agent acts optimally w.r.t. the reward $r_t = c_t^TW_t,$ we have that
 $ \underline{W}_t^T \Big(c_t\odot\Big(\mu^*_{c_t}-\mu^{\hat\pi_t}_{c_t}\Big)\Big) \leq 0$ . Combining these observations yields:
\begin{align}
\label{eq:1}
\left( \underline{W}^*-\underline{W}_t\right)^T \cdot \Big(c_t\odot\Big(\mu^*_{c_t}-\mu^{\hat\pi_t}_{c_t}\Big)\Big) > \epsilon > 0 \enspace .
\end{align}
This shows that $\underline{W}^*$ is never disqualified when updating $\Theta_t$. Since $\underline{W}^* \in \Theta_0$ this implies that $\forall t: \underline{W}^* \in \Theta_t$.
Now we show that not only $\underline{W}^*$ remains in the ellipsoid, but also a small ball surrounding it. If $\theta$ is disqualified by the algorithm: $ (\theta-\underline{W}_t)^T \cdot \Big(c_t\odot\Big(\mu^*_{c_t}-\mu^{\hat\pi_t}_{c_t}\Big)\Big) < 0 $ . Multiplying this inequality by -1 and adding it to \eqref{eq:1} yields: $(\underline{W}^*-\theta)^T \cdot \Big(c_t\odot\Big(\mu^*_{c_t}-\mu^{\hat\pi_t}_{c_t}\Big)\Big) > \epsilon.$ We apply H\"older inequality to LHS: 
\begin{align*}
    \epsilon & < \text{LHS} \\
    & \leq ||\underline{W}^*-\theta||_\infty \cdot || \Big(c_t\odot\Big(\mu^*_{c_t}-\mu^{\hat\pi_t}_{c_t}\Big)\Big) ||_1 \\
    & \leq \frac{2k}{1-\gamma}||\underline{W}^*-\theta||_\infty
\end{align*}
\noindent
Therefore for any disqualified $\theta$: $||\underline{W}^* - \theta||_\infty > \frac{(1-\gamma)\epsilon}{2k}$, thus $B_\infty\left(\underline{W}^*,\frac{(1-\gamma)\epsilon}{2k}\right)$ is never disqualified. This implies that: $$ \forall t: \text{vol}(\Theta_t) \geq \text{vol}(\Theta_0 \cap B_\infty(\underbar{$W^*$},\frac{(1-\gamma)\epsilon}{2k})) \geq \text{vol}(B_\infty(\underbar{$W^*$},\frac{(1-\gamma)\epsilon}{4k})) .$$ Finally, let $M_T$ be the number of rounds by $T$ in which $V^*_{c_t} - V^{\hat\pi_t}_{c_t} > \epsilon$. Using Lemma~\ref{lemma: ellipsoid} we get that:
\begin{align*}
\frac{M_T}{2(d k + 1)} \leq & \log \big(\text{vol}(\Theta_1)\big) - \log \big(\text{vol}(\Theta_{T+1})\big)\\
\leq & \log \big(\text{vol}\big(\text{MVEE}(B_\infty(0,1))\big)\big) - \log \big(\text{vol}(B_\infty(0,\frac{(1-\gamma)\epsilon}{4k}))\big) \\
\leq & \log \big(\text{vol}\big(\text{MVEE}(B_2(0,\sqrt{dk}))\big)\big) - \log \big(\text{vol}(B_2(0,\frac{(1-\gamma)\epsilon}{4k}))\big) \\
\leq & \log \big( ( \frac{4k\sqrt{dk}}{(1-\gamma)\epsilon} ) ^ {dk} \big)\\
\leq & d k\log\frac{4k\sqrt{d k}}{(1-\gamma)\epsilon} \enspace .
\end{align*}

\noindent 
Therefore $M_T \leq 2dk(dk+1)\log\frac{4k\sqrt{dk}}{(1-\gamma)\epsilon} = \mathcal{O}(d^2k^2\log(\frac{d k}{(1-\gamma)\epsilon}))$ .
\end{proof}

\subsection{Proof of Theorem~\ref{thm: trajectories}}
\begin{lemma}[Azuma's inequality]\label{lemma: azuma}
For a martingale $\{S_i\}_{i=0}^{n}$, if $|S_i-S_{i-1}|\le b$ a.s. for $i=1,...,n$:\\
$P\bigg(|S_n-S_0|>b\sqrt{2n\log(\frac{2}{\delta})}\bigg)<\delta$
\end{lemma}
\begin{proof}[Proof of Theorem~\ref{thm: trajectories}]
We first note that we may assume that for any $t$: $||W^*-W_t||_\infty \le 2$. If $\underbar{W}_t\not\in\Theta_0$, we update the ellipsoid by $\Theta_{t} \leftarrow \text{MVEE} \bigg(\bigg\{\theta\in\Theta_t:\Big(\theta-\underline{W}_t\Big)^T \cdot e_j \lessgtr 0 \bigg\}\bigg)$ where $e_j$ is the indicator vector of coordinate $j$ in which $\underbar{W}_t$ exceeds 1, and the inequality direction depends on the sign of $(\underbar{W}_t)_j$. If $\underbar{W}_t\not\in\Theta_0$ still, this process can be repeated for a finite number of steps until $\underbar{W}_t\in\Theta_0$, as the volume of the ellipsoid is bounded from below and each update reduces the volume (Lemma~\ref{lemma: ellipsoid}). Now we have $\underbar{W}_t\in\Theta_0$, implying $||W^*-W_t||_\infty \le 2$. As no points of $\Theta_0$ are removed this way, this does not affect the correctness of the proof. Similarly, we may assume $||W_t^*-W_t||_\infty \le 2$ as $W_t^* \in \Theta_0$.\\

\noindent
We denote $W_t$ which remains constant for each update in the batch by $W$. 
We define $t(i)$ the time-steps corresponding to the demonstrations in the batch for $i=1,...,n$.
We define $z_i^{*,H}$ to be the expected value of $\hat{z}_i^{*,H}$, and $z_i^*$ to be the outer product of $c_{t(i)}$ and the feature expectations of the expert policy for $W^*_{t(i)},c_{t(i)},\xi_{t(i)}'$ . We also denote $W^*_{t(i)}$ by $W_i^*$. We bound the following term from below, as in Theorem~\ref{thm: upper bound}:
\begin{align*}
& ( \underline{W}^*-\underline{W})^T \cdot (\frac{\Bar{Z}^*}{n}-\frac{\Bar{Z}}{n}) \\
& = \frac{1}{n}\sum_{i=1}^{n} (\underline{W}^*-\underline{W})^T \cdot ( \hat{z}_i^{*,H} - z_i) \\
& = \frac{1}{n}\sum_{i=1}^{n} (\underline{W}^*-\underline{W})^T \cdot ( z_i^* - z_i) + \frac{1}{n}\sum_{i=1}^{n} (\underline{W}^*-\underline{W})^T \cdot ( z_i^{*,H} - z_i^*) + \\
& \;\;\;\;\; \frac{1}{n}\sum_{i=1}^{n} (\underline{W}^*-\underline{W})^T \cdot ( \hat{z}_i^{*,H} - z_i^{*,H}) \\
& = \underbrace{\frac{1}{n}\sum_{i=1}^{n} (\underline{W}_i^*-\underline{W})^T \cdot ( z_i^* - z_i)}_{(1)} + \underbrace{\frac{1}{n}\sum_{i=1}^{n} (\underline{W}^*-\underline{W}_i^*)^T \cdot ( z_i^* - z_i)}_{(2)} + \\
& \;\;\;\; \underbrace{\frac{1}{n}\sum_{i=1}^{n} (\underline{W}^*-\underline{W})^T \cdot ( z_i^{*,H} - z_i^*)}_{(3)} + \underbrace{\frac{1}{n}\sum_{i=1}^{n} (\underline{W}^*-\underline{W})^T \cdot ( \hat{z}_i^{*,H} - z_i^{*,H})}_{(4)}
\end{align*}

\noindent
\textbf{(1)}: Since the sub-optimality criterion implies a difference in value of at least $\epsilon$ for the initial distribution which assigns $1$ to the state where the agent errs, we may use identical arguments to the previous proof. Therefore, the term is bounded from below by $\epsilon$.\\

\noindent
\textbf{(2)}: By assumption $|| \underline{W}^*-\underline{W}_i^*||_\infty \le \frac{(1-\gamma)\epsilon}{8k}$ thus since $||( z_i^* - z_i)||_1 \le \frac{2k}{1-\gamma}$ by H\"older's inequality the term is bounded by $\frac{\epsilon}{4}$.\\

\noindent
\textbf{(3)}: We have $||x_i^{*,H}-x_i^*||_1 \le \frac{k\gamma^H}{1-\gamma}$ from definitions, thus $||z_i^{*,H}-z_i^*||_1 \le \frac{k\gamma^H}{1-\gamma}$ since $c \in \Delta_{d-1}$. As mentioned previously we may assume $||W^*-W_t||_\infty \le 2$, therefore by H\"older's inequality the term is bounded by $\frac{\epsilon}{4}$ due to our choice of $H$:
\begin{align*}
\gamma^H & = (1-(1-\gamma))^H \\
& = \big((1-(1-\gamma))^{\frac{1}{1-\gamma}}\big)^{(1-\gamma)H}\\
& = \big((1-(1-\gamma))^{\frac{1}{1-\gamma}}\big)^{\log(\frac{8k}{(1-\gamma)\epsilon})}\\
& \le e^{-\log(\frac{8k}{(1-\gamma)\epsilon})}\\
& = \frac{(1-\gamma)\epsilon}{8k}.
\end{align*}

\noindent
\textbf{(4)}: 
The partial sums $\sum_{i=1}^N(\underline{W}^*-\underline{W})^T \cdot (z_i^{*,H} - \hat{z}_i^{*,H})$ for $N=0,...,n$ form a martingale sequence. Note that: $$||z_i^{*,H}||_1 \le \frac{k}{1-\gamma},\;\; ||\hat{z}_i^{*,H}||_1 \le \frac{k}{1-\gamma},\;\; ||W^*-W_t||_\infty \le 2,$$ thus, we can apply Azuma's inequality (Lemma~\ref{lemma: azuma}) with $b=\frac{4k}{(1-\gamma)}$ and with our chosen n this yields: $\sum_{i=1}^n(\underline{W}^*-\underline{W})^T \cdot (z_i^{*,H} - \hat{z}_i^{*,H}) \le \frac{n\epsilon}{4}$ with probability of at least $1-\frac{\delta}{2dk(dk+1)\log(\frac{16k\sqrt{dk}}{(1-\gamma)\epsilon})}$.\\

\noindent
Thus $(\underline{W}^*-\underline{W})^T \cdot (\frac{\Bar{Z}^*}{n}-\frac{\Bar{Z}}{n}) > \frac {\epsilon}{4}$ and as in Theorem~\ref{thm: upper bound} this shows $B_\infty(\underline{W}^*,\frac{(1-\gamma)\epsilon}{8k})$ is never disqualified, and the number of updates is bounded by $2dk(dk+1)\log(\frac{16k\sqrt{dk}}{(1-\gamma)\epsilon})$, and multiplied by n this yields the upper bound on the number of rounds in which a sub-optimal action is chosen. By union-bound, the required bound for term \textbf{(4)} holds in all updates with probability of at least $1-\delta$.
\end{proof}

\newpage
\section{Proofs for Section 3.4}
\begin{proof}[Proof of Theorem ~\ref{thm: transfer}]
This proof follows the proof for Lemma 1 in \citep{barreto2017successor}, with additional arguments taken from proofs of the simulation lemma. We define $\epsilon_P = \max_{s,a}||P_c(\cdot|s,a)-P_{c_j}(\cdot|s,a)||_1, \epsilon_R = \max_{s} |R^*_c(s) - R^*_{c_j}(s)|$. \\
We first note that:
\begin{equation*}
\begin{split}
Q^*_{c}(s,a) - Q^{\pi^*_{c_j}}_{c}(s,a) & \le |Q^*_{c}(s,a) - Q^*_{c_j}(s,a)| + |Q^*_{c_j}(s,a) - Q^{\pi^*_{c_j}}_{c}(s,a)| \\
\end{split}
\end{equation*}
and bound each of these terms:
\begin{equation*}
\begin{split}
& |Q^*_{c}(s,a) - Q^*_{c_j}(s,a)| \\
& = \big|R^*_c(s) - R^*_{c_j}(s) + \gamma \big( \sum_{s'} P_c(s'|s,a) \max_b Q^*_{c}(s',b) -  \sum_{s'} P_{c_j}(s'|s,a) \max_b Q^*_{c_j}(s',b) \big) \big| \\
& \le |R^*_c(s) - R^*_{c_j}(s)| + \gamma \big| \sum_{s'} P_c(s'|s,a) \big( \max_b Q^*_{c}(s',b) -  \max_b Q^*_{c_j}(s',b) \big) \big| + \\
& \;\;\;\;\; \gamma \big| \sum_{s'} \big(P_c(s'|s,a) - P_{c_j}(s'|s,a) \big) \max_b Q^*_{c_j}(s',b) \big| \\
& \le \epsilon_R + \gamma \sum_{s'} P_c(s'|s,a) \big| \max_b Q^*_{c}(s',b) -  \max_b Q^*_{c_j}(s',b) \big| +  \gamma V_{max} ||P_c(\cdot|s,a)-P_{c_j}(\cdot|s,a)||_1\\
& \le \epsilon_R + \gamma V_{max} \epsilon_P + \gamma \sum_{s'} P_c(s'|s,a) \max_b \big| Q^*_{c}(s',b) - Q^*_{c_j}(s',b) \big|\\
& \le \epsilon_R + \gamma V_{max} \epsilon_P + \gamma \max_{s',b} | Q^*_{c}(s',b) - Q^*_{c_j}(s',b) |
\end{split}
\end{equation*}
taking $\max_{s,a}$ of the resulting inequality and solving for LHS yields:
$$ \max_{s,a}|Q^*_{c}(s,a) - Q^*_{c_j}(s,a)| \le \frac{\epsilon_R + \gamma V_{max} \epsilon_P}{1-\gamma}. $$
For the second term we follow similar steps:
\begin{equation*}
\begin{split}
& |Q^*_{c_j}(s,a) - Q^{\pi^*_{c_j}}_{c}(s,a)| \\
& = \epsilon_R + \gamma \big| \sum_{s'} P_c(s'|s,a) Q^*_{c_j}(s',\pi^*_{c_j}(s')) - \sum_{s'} P_{c_j}(s'|s,a)  Q^{\pi^*_{c_j}}_{c}(s',\pi^*_{c_j}(s')) \big|\\
& \le \epsilon_R + \gamma V_{max} \epsilon_P + \gamma \sum_{s'} P_c(s'|s,a) \big| Q^*_{c_j}(s',\pi^*_{c_j}(s')) - Q^{\pi^*_{c_j}}_{c}(s',\pi^*_{c_j}(s')) \big| \\
& \le \epsilon_R + \gamma V_{max} \epsilon_P + \gamma max_{s',b} | Q^*_{c_j}(s',b) - Q^{\pi^*_{c_j}}_{c}(s',b) |
\end{split}
\end{equation*}
taking $\max_{s,a}$ of the resulting inequality and solving for LHS yields:
$$ \max_{s,a}|Q^*_{c_j}(s,a) - Q^{\pi^*_{c_j}}_{c}(s,a)| \le \frac{\epsilon_R + \gamma V_{max} \epsilon_P}{1-\gamma}. $$
Plugging these bounds into the first inequality yields:
$$ Q^*_{c}(s,a) - Q^{\pi^*_{c_j}}_{c}(s,a) \le 2 \frac{\epsilon_R + \gamma V_{max} \epsilon_P}{1-\gamma}. $$
Now, we express $\epsilon_R, \epsilon_P$ in terms of the distance between the contexts:
\begin{align*}
\epsilon_P & = \max_{s,a} \sum_{s'} |P_c(s'|s,a)-P_{c_j}(s'|s,a)| \\
& = \max_{s,a} \sum_{s'} |(c-c_j)^T \begin{bmatrix} P_1(s'|s,a) \\ \vdots \\ P_d(s'|s,a) \end{bmatrix} | \\
& \le ||c-c_j||_\infty \sum_{s'}|| \begin{bmatrix} P_1(s'|s,a) \\ \vdots \\ P_d(s'|s,a) \end{bmatrix} ||_1\\
& = d||c-c_j||_\infty ,\\
\epsilon_R & = \max_s|R^*_c(s) -R^*_{c_j}(s)|\\
& = \max_s|(c-c_j)^T(W^*\phi(s)| \\
& \le ||c-c_j||_\infty \max_s ||W^*\phi(s)||_1 \\
& = \phi_{max}||c-c_j||_\infty
\end{align*}
\noindent
which, plugged into our inequality, yields:
$$ Q^*_{c}(s,a) - Q^{\pi^*_{c_j}}_{c}(s,a) \le 2 \frac{\phi_{max} + \gamma d V_{max}}{1-\gamma}||c-c_j||_\infty. $$
Note that as a special case, if the dynamics are identical for all contexts, $\epsilon_P=0$, therefore:
$$ Q^*_{c}(s,a) - Q^{\pi^*_{c_j}}_{c}(s,a) \le 2 \frac{\phi_{max}}{1-\gamma}||c-c_j||_\infty. $$
To convert the bound to the value function, we add a dummy initial state $s_0$, with $\phi(s_0)=0$ and $\forall a: P(\cdot|s_0,a)=\xi$. In this case, applying the above inequality for the initial state yields:
$$ V^*_{c} - V^{\pi^*_{c_j}}_{c} = \frac{1}{\gamma}\big( Q^*_{c}(s_0,a) - Q^{\pi^*_{c_j}}_{c}(s_0,a) \big) \le 2 \frac{\phi_{max} + \gamma d V_{max}}{\gamma(1-\gamma)}||c-c_j||_\infty $$
\end{proof}
\begin{proof}[Proof of Theorem ~\ref{thm: gpi}]
We denote the maximizing index and action by $$a_i, i \in \argmax_a \argmax_i Q^{\pi^*_i}_{c}(s,a).$$ We have that
\begin{align*}
V^*_{c}(s) - V^\pi_{c}(s) & = \max_a Q^*_{c}(s,a) - Q^{\pi}_{c}(s,a_i)\\
& \le \max_a Q^*_{c}(s,a) - Q^{\pi_{c_i}^*}_{c}(s,a_i)\\
& \le \max_a Q^*_{c}(s,a) - \max_a Q^{\pi_{c_j}^*}_{c}(s,a) \\
& \le \max_a  \big(Q^*_{c}(s,a) - Q^{\pi_{c_j}^*}_{c}(s,a)\big) \\
& \le 2 \frac{\phi_{max}}{1-\gamma}||c-c_j||_\infty
\end{align*}
where the first inequality is due to the Generalized Policy Improvement Theorem, Theorem 1 in \citep{barreto2017successor}, which claims: $Q^\pi_{c}(s,a) \ge Q^{\pi_{c_j}^*}_{c}(s,a)$, the second inequality is from the definition of $i,a_i$, and the last inequality is the second to last inequality from the previous proof. Taking $\min_j$ then expectation w.r.t. $s \sim \xi$ finishes this proof.
\end{proof}

\section{Experiments}\label{app:exp}
In this section, we describe the technical details of our experiments, including the hyper-parameters used. To solve MDPs, we use value iteration. Our implementation is based on a stopping condition with a tolerance threshold, $\tau$, such that the algorithm stops if $|V_t-V_{t-1}| < \tau.$ In the autonomous driving simulation and grid world domains we use $\tau = 10^{-4}$ and in the dynamic treatment regime we use $\tau = 10^{-3}$. 
\\


\subsection{Grid world}

The grid world domain, presented at \cref{subsec:largemdp} is constructed for computational comparisons between methods. The test data includes $100$ contexts. Here we include more results in this domain, all measured on the same setup as in \cref{subsec:largemdp}.
\\

\begin{figure*}[h]
    \centering
    \subfigure[Value]{\includegraphics[width=0.472\linewidth]{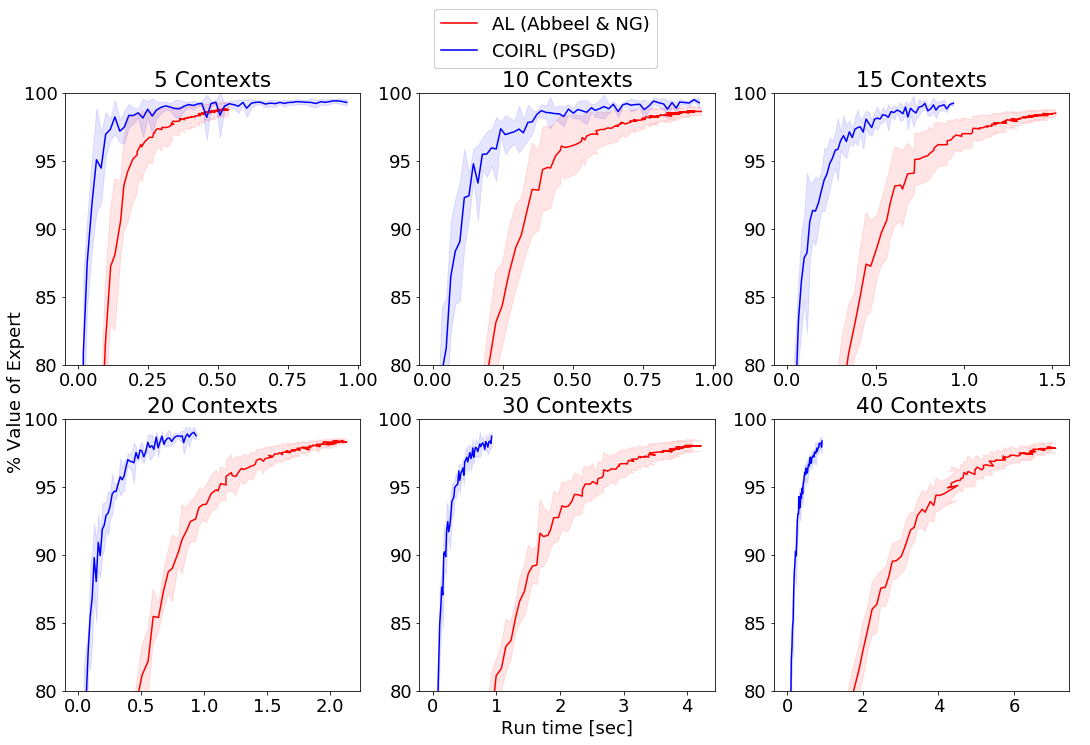}
    \label{fig: AL_valueXruntime}}
    \subfigure[Accuracy]{\includegraphics[width=0.472\linewidth]{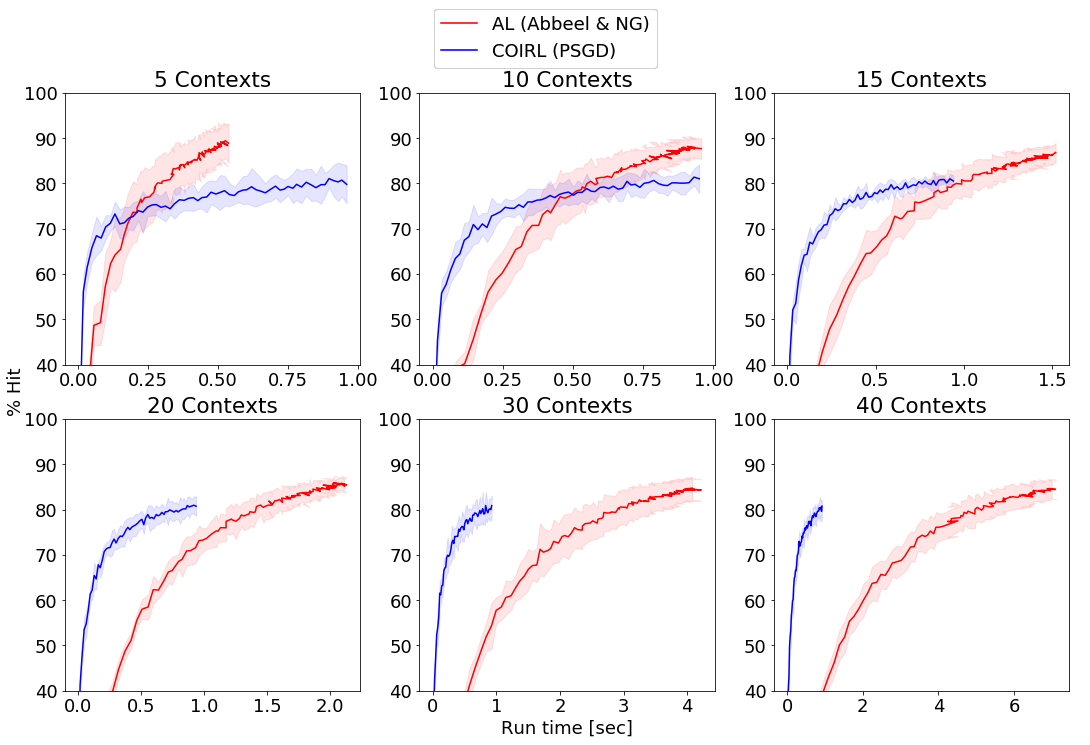}
    \label{fig: AL_accuracyXruntime}}
    \caption{\textbf{value and accuracy as a function of run-time in various context space size}. The advantage of COIRL grows as the context space grows.}
\end{figure*}

\noindent
The results shown in \cref{fig: AL_valueXruntime} present the value as a function of run-time for each context space size $|\mathcal{C}|$. In \cref{fig: AL_accuracyXruntime} we show the accuracy w.r.t. the expert policy in the same manner. We observe that COIRL achieves better value for any size of the train data and that AL achieves better accuracy after convergence. The accuracy over run-time show that the convergence of AL in a large context space takes more time and that the accuracy gap between the methods after convergence is reduced.

\subsection{Autonomous driving simulation}

The environment is modeled as a tabular MDP that consists of $1531$ states. The speed is selected once, at the initial state, and is kept constant afterward. The other $1530$ states are generated by 17 X-axis positions for the agent's car, 3 available speed values, 3 lanes and 10 Y-axis positions in which car B may reside. During the simulation, the agent controls the steering direction of the car, moving left or right, i.e., two actions. The feature vector $\phi(s)$ is composed of 3 features: (1) speed, (2) ``collision", which is set to $0$ in case of a collision and $0.5$ otherwise, and (3) ``off-road", which is $0.5$ if the car is on the road and $0$ otherwise.
\\

\noindent
In these experiments, we define our mappings in a way that induces different behaviors for different contexts, making generalization a more challenging task. We evaluate all algorithms on the same sequences of contexts, and average the results over $5$ such sequences. The test data in this domain contains $80$ unobserved contexts.

\subsubsection{Ellipsoid setting}\label{section:ellipsoid_experiments}

This section describe the technical details about the experiments in \cref{subsec:ellipsoid_setting}. Here, the real mapping between the context to the reward is \textbf{linear}. We define $W^* = \Big( \begin{smallmatrix} -1 & 0.75 &  0.75 \\  0.5 & -1 & 1 \\ 0.75 & 1 & -0.75 \end{smallmatrix} \Big)$, before normalization. The contexts are sampled uniformly in the 2-dimensional simplex.
\\

\noindent
\textbf{Hyper-parameter selection and adjustments}: 
The algorithms maintained a $3\times3$ matrix to estimate $W^*$.

\noindent
\emph{Ellipsoid:} By definition, the ellipsoid algorithm is hyper-parameter free and does not require tuning.

\noindent
\emph{PSGD:} The algorithm was executed with with the parameters: $\alpha_0=0.3,\alpha_t=0.9^t\alpha_{t-1}$, and iterated for 40 epochs. An outer decay on the step size was added for faster convergence, the initial $\alpha_0$ becomes $0.94\cdot \alpha_0$ every time a demonstration is presented. The gradient, $g_t$ is normalized to be $g_t=g_t\frac{g_t}{||g_t||_\infty}$ and the calculated step is taken if: $cW_t\big(\mu(\hat{\pi}^t_c)-\mu(\pi^*_c)\big) > cW_{t+1}\big(\mu(\hat{\pi}^{t+1}_c)-\mu(\pi^*_c)\big) $, where $\hat{\pi}^t_c$ denotes the optimal policy for a context $c$ according to $W_t$.

\noindent
\emph{ES:} The algorithm was executed with the parameters: $\sigma=10^{-3}, m=250, \alpha = 0.1$ with decay rate of $0.95$, for $50$ iterations which didn't iterate randomly over one the contexts, but rather used the entire training set (all of the observed contexts and expert demonstrations up to the current time-step) for each step.
The matrix was normalized according to $||\cdot||_2$, and so was the step calculated by the ES algorithm, before it was multiplied by $\alpha$ and applied.

\subsubsection{Online setting}
This section describes the experiments of the online setting (\cref{subsub:car_online}).
All of the compared methods minimize the same objective, where the subgradients for the descent direction are computed using either the feature expectations (\textbf{feature expectations} setup) or expert trajectories of length $40$ (\textbf{trajectories} setup). In this framework at every iteration we sample one context and its corresponding feature expectations (or trajectory, sampled from the expert policy), and take one descent step according to it. The mapping from context to reward $W$ is linear, and projected to the EW algorithm requirement to be in the $dk-1$ simplex. In the autonomous driving simulation: $W^* = \Big( \begin{smallmatrix} 0.043 & 0 &  0.043 \\  0 & 0.434 & 0 \\ 0.043 & 0.434 & 0 \end{smallmatrix} \Big)$.
\\

\noindent
\textbf{Hyper-parameter selection and adjustments}:
The PSGD and EW algorithms are configured as the theory specifies, where each descent step is calculated from the one sample. The ES algorithm is applied with the parameters $\sigma = 10^{-3}, m=500, \alpha = 0.1$ with decay rate $0.95$, for every iteration. The ES implementation include a special enhancement; a descent step is taken if the objective function value decreases (after the descent step).


\subsection{Sepsis treatment}

For the sepsis treatment we construct two environments, one for simulation purposes - simulator, and another for evaluation on real-life data. The experiments with the simulator presented in \cref{subsec:med_rew}, the evaluation on real-life data presented in \cref{subsec:med_real}.
\\

\subsubsection{Sepsis treatment simulator}
As described in \cref{subsec:med_rew} we use the processed data from \cite{1902.03271}. It consists of $5366$ trajectories, each representing the sequential treatment provided by a clinician to a patient. At each time-step, the available information for each patient consists of $8$ static measurements and $41$ dynamic measurements. In addition, each trajectory contains the reported actions performed by the clinician (the number of fluids and vasopressors given to a patient at each time-step and binned to 25 different values), and there is a mortality signal which indicates whether the patient was alive 90 days after his hospital admission.
\\

\noindent
In order to create a tabular CMDP from the processed data, we separate the static measurements of each patient and keep them as the context. We cluster the dynamic measurements using K-means \citep{macqueen1967some}. Each cluster is considered a state and the coordinates of the cluster centroids are taken as its features $\phi(s)$. We construct the transition kernel between the clusters using the empirical transitions in the data given the state and the performed actions. Two states are added to the MDP and the feature vector is extended by $1$ element, corresponding to whether or not the patient died within the $90$ days following hospital release. This added feature receives a value of $0$ on all non-terminal states, a value of $-0.5$ for the state representing the patient's death and $0.5$ for the one representing survival. In addition, as the number of trajectories is limited, not all state-action pairs are represented in the data. In order to ensure the agent does not attempt to perform an action for which the outcome is unknown, we add an additional terminal state. At this state, all features are set to $-1$ to make it clearly distinguishable from all other states in the CMDP.
\\

\noindent
In our simulator, we used the same structure as the raw data, i.e., we used the same contexts found in the data and the same initial state distribution. Each context is projected onto the simplex and the expert's feature expectations for each context are attained by solving the CMDP. While we focus on a simulator, as it allows us to analyze the performance of the algorithms, our goal is to have a reward structure which is influenced by the data. Hence, we produce $W^*$ by running the ellipsoid algorithm on trajectories obtained from the data.
As done in the autonomous driving simulation, we average our results over $5$ different seeds. The test data size in this domain is $300$.
\\


\subsubsection{Online setting}
Similarly to the autonomous driving simulation, there are two setups. For the \textbf{trajectories} setup we use expert trajectories of length $40$. Again, the mapping from context to reward $W$ is linear, and projected to the EW algorithm requirement to be in the $dk-1$ simplex (the true mapping can be found in the supplementary code).
\\

\noindent
\textbf{Hyper-parameter selection and adjustments}:
The PSGD method is configured as specified by the theory, where each descent step is calculated from one sample.
\\

\subsubsection{Offline setting}
The offline setting evaluates the methods' performance on a limited train data set. In this framework, at every iteration we sample a mini-batch of contexts (from a finite set) and their corresponding trajectories (sampled from the expert policy) then taking one descent step according to them. 
We conduct two experiments that evaluate the performance on a fixed-size data set. First, we consider a \textbf{linear} mapping, followed by an analysis of the convergence when a DNN estimator of the reward is used, when the mapping is \textbf{non-linear}. Of the various COIRL methods, for these experiments, we focus on PSGD, as it is less restrictive on $\mathcal{W}$. In all experiments the train data consists of pairs of context and feature expectations of a trajectory of length $40$.
\\

\noindent
We evaluate the PSGD method and the GPI method (using the mapping calculated by PSGD) along with BC. The evaluation is done after convergence on a changing train data size, measured as 'contexts', which refer to the number of expert trajectories given (one per context). In the non-linear setting The non-linear model of PSGD implemented by a DNN with the context as its input, three layers with a leakyReLU activation and batch-normalization, each one of size $336$. BC in both environments implemented by a DNN that has three layers of sizes $250$,$125$,$25$ respectively, a leakyReLU activation between the first and second layers and a Softmax activation on the output to ensure a probability vector.
\\

\noindent
\textbf{Hyper-parameter selection and adjustments}: 
In the \textbf{linear} setting the PSGD algorithm is configured with step-size $\alpha_t = 0.25\cdot 0.95^t$, the mini-batch size is $10$, similarly to the autonomous driving simulation. In this domain the stopping criteria is $60$ iterations.
\\

\noindent
The \textbf{non-linear} setting computes the descent direction by backpropagation of the subgradient. Each descent step is calculated over a mini-batch of size $32$, where the step-size is $\alpha_t = 0.3\cdot 0.96^t$. We measure the results for 200 iterations. The train data consists $4000$ contexts. The true mapping is defined by $f^*(c) = r_1 \text{ if} \enspace \text{age} > 0.1, \enspace \text{and} \enspace r_2 \text{ otherwise}$, where \textbf{age} refers to the normalized age of the patient, an element of the context vector.

\end{document}